\documentclass{article}

\PassOptionsToPackage{numbers,sort&compress}{natbib}
\usepackage[preprint]{neurips_2026}

\usepackage[utf8]{inputenc}
\usepackage[T1]{fontenc}
\usepackage{hyperref}
\hypersetup{hidelinks}
\usepackage{url}
\usepackage{booktabs}
\usepackage{longtable}
\usepackage{amsfonts}
\usepackage{amsmath}
\usepackage{amssymb}
\usepackage{amsthm}
\usepackage{nicefrac}
\usepackage{microtype}
\usepackage{xcolor}
\usepackage{graphicx}
\usepackage{array}
\usepackage{titlesec}
\usepackage{soul}

\definecolor{linkblue}{RGB}{0,0,200}
\newcommand{\codelink}[1]{\href{#1}{\textcolor{linkblue}{\ul{link}}}}

\newtheorem{theorem}{Theorem}
\newtheorem{proposition}{Proposition}

\theoremstyle{plain}
\newtheorem*{restatement}{Restatement}

\titlespacing*{\section}{0pt}{4pt plus 1pt minus 1pt}{2pt plus 1pt}
\titlespacing*{\subsection}{0pt}{3pt plus 1pt minus 1pt}{1pt plus 1pt}
\titlespacing*{\paragraph}{0pt}{1pt plus 1pt minus 0pt}{0.5em}

\title{CAST: Causal Anchored Simplex Transport\\for Distribution-Valued Time Series}

\author{%
  Jiecheng Lu\\
  Georgia Institute of Technology\\
  \texttt{jlu414@gatech.edu}\\
  \And
  Jieqi Di\\
  Georgia Institute of Technology\\
  \texttt{jdi7@gatech.edu}\\
  \AND
  Runhua Wu\\
  Georgia Institute of Technology\\
  \texttt{rwu330@gatech.edu}\\
  \And
  Yuwei Zhou\\
  Indiana University\\
  \texttt{yuwzhou@iu.edu}\\
}

\begin{document}

\maketitle

\begin{abstract}
Many decision-facing stochastic systems are observed through aggregate distributions rather than scalar trajectories: queue occupancies, mobility shares, public-health mixtures, generation-source shares, ecological compositions, and air-quality severity profiles all live on the probability simplex and evolve over time. We study causal (time-respecting online) forecasting for these distribution-valued time series and argue that the transition operator itself should be structured around the simplex. We introduce CAST (Causal Anchored Simplex Transport), a successor-local operator that (i) retrieves empirical successors from causal context, (ii) stabilizes them with a persistence anchor, and (iii) applies a bounded local stochastic transport on ordered supports; every stage preserves the simplex by construction. We identify a structural failure mode, latent transition-kernel aliasing, where similar observed distributions evolve differently under different contextual regimes, and prove that any forecaster depending only on an aliased summary incurs an irreducible weighted Jensen--Shannon excess-risk lower bound, while the CAST hypothesis class contains the regime-aware Bayes successor; for ordered supports an additional Pinsker separation holds whenever the transported successor lies outside the no-transport anchor hull. On a suite of eleven public and simulated benchmarks spanning ecology, energy, diet, mortality, employment, air quality, severe weather, mobility, and $G/G/1$, $G_t/G/1$ queue occupancy, CAST achieves the best average rank on both one-step KL (1.27) and autoregressive rollout JSD (1.91), winning 8/11 sections on each metric against a broad statistical, compositional, recurrent, convolutional, Transformer, and modern time-series baseline set, and top-2 on all 11 sections for offline KL. Component ablations and a controlled synthetic aliasing experiment corroborate the theory. The code release is available at this~\codelink{https://anonymous.4open.science/r/Causal-Anchored-Simplex-Transport-8775}.

\end{abstract}

\section{Introduction}
\label{sec:introduction}

Modern stochastic systems are increasingly observed, controlled, and planned through aggregate distributions rather than isolated events or scalar trajectories. The rapid deployment of large language model services has in particular turned cloud inference into a high-throughput stochastic system with heterogeneous request sizes, variable service times, memory and cache constraints, and batching effects, placing queueing and scheduling at the center of high-performance inference \citep{mitzenmacher2025queueing}. For capacity planning, rebalancing, and risk-sensitive control, what matters is how the distribution of system states---queue length, workload, congestion regime---evolves over the next several intervals. The same pattern recurs across mobility platforms \citep{nyctlc_trip_records}, public-health panels, power systems, ecological studies, and air-quality severity reports.

\paragraph{Distribution-valued forecasting.} We formalize this by letting
\[
    p_t \in \Delta^{D-1}
    := \Bigl\{p\in\mathbb{R}_{\ge 0}^{D}:\textstyle\sum_{j=1}^{D} p(j)=1\Bigr\}
\]
denote a probability vector observed at time $t$, and studying the causal one-step map $\hat p_{t+1}=f(p_{1:t})$ that uses only present and past distributions. Here causal means time-respecting online prediction, not interventional causal inference. This is a compositional time series problem: observations are not arbitrary Euclidean vectors but proportions of a whole, and the simplex is the natural sample space \citep{aitchison1982statistical,aitchison1986statistical}. Classical compositional forecasters apply log-ratio transformations, state-space recursions, exponential smoothing, or VAR dynamics in unconstrained coordinates and then map back to the simplex \citep{snyder2015forecasting,snyder2017forecasting,barcelo2011compositional}; Wasserstein autoregressive and functional distributional models treat probability distributions as serially dependent objects \citep{zhang2022wasserstein,jiang2022wasserstein}. These provide important statistical baselines but leave a modeling question open: can the transition operator itself be structured as a causal map on the simplex?

\paragraph{Why the full distribution.} Forecasting the whole distribution provides a signal that a single trajectory forecast cannot. A future window may place substantial mass on both a light-load and a severe-congestion regime, or split across compositions with very different operational consequences. Distribution-valued forecasts are therefore the natural primitive for early warning, resource allocation, risk-sensitive planning, and distributionally robust optimization \citep{bellemare2017distributional,li2022quantile,chow2015cvar,wiesemann2014dro}. Queueing makes this concrete: a discrete-time single-server queue obeys $L_{t+1}=\max\{0,\,L_t+A_t-C_t\}$ \citep{lindley1952theory,harchol2013performance,palomo2020learning}, and recent work casts next-step queue-length distribution prediction as its own forecasting target on simulated $G/G/1$ and $G_t/G/1$ systems \citep{di2025transformerqueue}. Setting $p_t(k)\approx\mathbb P(L_t=k)$ turns the problem into one of forecasting how probability mass redistributes across an ordered support---directly exposing the probability of idleness, congestion, or overflow.

\paragraph{The structural challenge: latent-kernel aliasing.} In distribution-valued dynamics the same observed composition can arise from different latent transition mechanisms: a queue occupancy histogram may look identical while the system is entering high load or draining after a burst; a similar air-quality histogram may appear under different seasonal or regional regimes; similar ecological compositions may have different successors because the sites have different histories. In such cases the current distribution alone does not determine the next distribution, and a forecaster depending only on a fixed visible summary must collapse multiple latent kernels into one averaged successor. Existing deep time-series architectures---recurrent, convolutional, attention-based, residual-MLP, patching, inverted, multiscale, frequency-decomposed, cross-dimension, channel-mixing, exogenous-aware, foundation-model \citep{salinas2020deepar,lim2021temporal,oreshkin2020nbeats,challu2023nhits,zhou2021informer,wu2021autoformer,zhou2022fedformer,zeng2023transformers,nie2023time,wu2023timesnet,zhang2023crossformer,das2023tide,chen2023tsmixer,lu2024arm,liu2024itransformer,wang2024timexer,lu2024cats,wang2024timemixer,lu2025wave,ansari2024chronos,lu2026hypermlp,das2024timesfm,kim2026stretchtime}---can be cross-applied to probability vectors, but their output heads enforce validity through a final softmax and do not encode empirical successor locality, latent-regime disambiguation, or ordered mass motion.

\begin{figure}[t]
    \centering
    \includegraphics[width=0.96\textwidth]{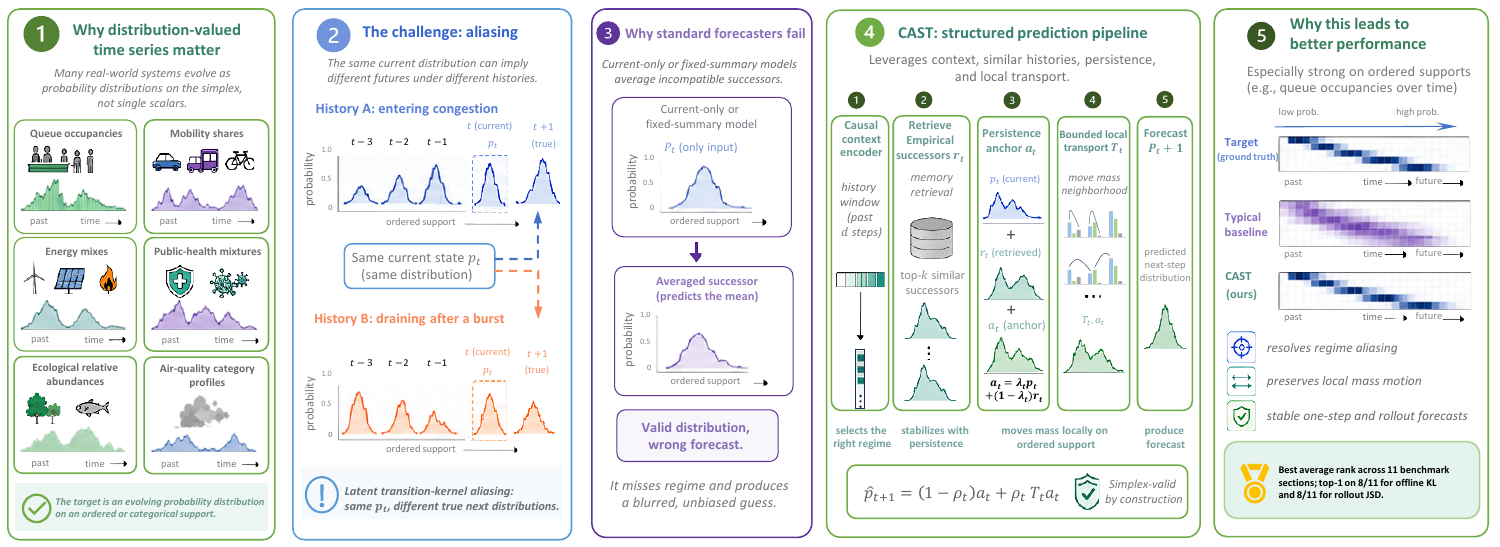}
    \caption{\textbf{CAST overview.} (1) Real-world systems evolve as distributions on a simplex rather than scalar trajectories; (2) the same $p_t$ can arise from histories with different successors, so (3) current-only forecasters collapse to a mixture average; (4) CAST encodes history, retrieves empirical successors $r_t$, forms a persistence--retrieval anchor $a_t=\lambda_t p_t+(1-\lambda_t)r_t$, and applies bounded local transport $T_t$ on ordered supports, $\hat p_{t+1}=(1-\rho_t)a_t+\rho_t T_t a_t$, simplex-valid by construction; (5) this resolves regime aliasing and stabilizes rollout.}
    \label{fig:cast-overview}
\end{figure}

\paragraph{CAST.} We introduce CAST, a causal successor-local operator on the simplex (Figure~\ref{fig:cast-overview}). Given the causal history, CAST retrieves empirical successors from past context-successor pairs, forms a persistence-successor anchor, and, when the support is ordered, applies a bounded local stochastic transport operator:
\[
    \hat p_{t+1}
    =
    (1-\rho_t)\,a_t+\rho_t\,T_t a_t,
    \qquad
    a_t=\lambda_t p_t+(1-\lambda_t)\,r_t,
\]
where $r_t$ is a retrieved causal successor, $a_t$ is a persistence--retrieval anchor, $T_t$ is a radius-bounded row-stochastic transport, and $\rho_t$ is a bounded transport strength. Every stage preserves the simplex by construction. This structure encodes a simple inductive bias: future distributions should be close to empirical successors observed in similar causal contexts, stabilized by persistence and corrected only by small local mass motion when the support has a natural order. Retrieval addresses aliasing (two visible distributions with different histories have different successors); transport addresses ordered mass motion (on queue occupancies, AQI levels, ages, severity bins, probability mass should move locally, not through an unconstrained global correction).

\paragraph{Contributions.} (i) Model: we propose CAST, a causal successor-local transition model for distribution-valued time series, whose prediction is a convex combination of an empirical-successor anchor and a bounded local stochastic transport of that anchor on ordered supports. (ii) Theory: we prove a weighted Jensen--Shannon aliasing lower bound that any fixed-summary forecaster must pay when multiple latent regimes share the same observed summary, show that the CAST hypothesis class contains the regime-aware Bayes successor under mild memory assumptions, and give an ordered-support Pinsker separation for transports that leave the anchor hull; approximation bounds under imperfect retrieval/transport and simplex-closure/rollout-drift control complete the picture. (iii) Benchmarks: on an 11-section benchmark suite spanning ecology, energy, diet, mortality, employment, air quality, severe weather, mobility, and queueing, CAST obtains average ranks of 1.27 (offline KL) and 1.91 (rollout JSD), winning 8/11 sections on each metric against statistical, compositional, recurrent, convolutional, Transformer, and modern time-series baselines, and is top-2 on all 11 sections for offline KL. A targeted five-seed robustness study, a component ablation, and a controlled synthetic aliasing experiment support the theory-driven reading of the results.

\section{Related Work}
\label{sec:related-work}

\paragraph{Compositional and distributional time series.} Aitchison's log-ratio program \citep{aitchison1982statistical,aitchison1986statistical} treats the simplex as the sample space for compositions; classical methods fit VAR, state-space, or exponential-smoothing recursions in ilr/alr coordinates and map back \citep{snyder2015forecasting,snyder2017forecasting,barcelo2011compositional}. Distribution time series have also been studied via Wasserstein autoregressive and functional-data models \citep{zhang2022wasserstein,jiang2022wasserstein,wang2025koopman}, with optimal-transport tools for geometry on probability measures \citep{cuturi2013sinkhorn,peyre2019computational}. These approaches motivate the object of study but do not jointly implement causal successor retrieval, persistence anchoring, and ordered local transport.

\paragraph{Deep sequence forecasting.} Modern time-series architectures span recurrent, attention-based multi-horizon, residual-MLP, hierarchical interpolation, Transformer, patching, periodic decomposition, inverted, multiscale mixing, frequency-decomposed, cross-dimension, channel-mixing, exogenous-aware, in-context, language-model-reprogramming, and foundation-model approaches \citep{salinas2020deepar,lim2021temporal,oreshkin2020nbeats,challu2023nhits,zhou2021informer,wu2021autoformer,zhou2022fedformer,zeng2023transformers,nie2023time,wu2023timesnet,zhang2023crossformer,das2023tide,chen2023tsmixer,lu2024arm,liu2024itransformer,wang2024timexer,lu2024cats,wang2024timemixer,lu2025incontext,jin2024timellm,lu2025wave,ansari2024chronos,lu2026fem,woo2024moirai,das2024timesfm,pati2026caps}, with concurrent progress on efficient and linear-attention sequence backbones that scale subquadratically in context length \citep{katharopoulos2020transformers,choromanski2021performer,gu2024mamba,yang2024gla,lu2025lineartransformers,sun2023retnet,lu2025zeros}. With a simplex output head they are our neural baselines, but they learn arbitrary output logits rather than a transition operator structured around successor locality, regime disambiguation, or ordered mass motion. Neural temporal point processes \citep{du2016recurrent,mei2017neural,shchur2021neural} instead predict future event times and marks.

\paragraph{Queueing and risk-sensitive decisions.} Queueing theory supplies the motivating example of distribution-valued dynamics on an ordered support \citep{lindley1952theory,harchol2013performance,mitzenmacher2025queueing,palomo2020learning}; our queue benchmark follows the simulation protocol of \citet{di2025transformerqueue}. Downstream, distributional RL models return distributions \citep{bellemare2017distributional}, quantile/CVaR MDPs optimize risk-sensitive summaries \citep{li2022quantile,chow2015cvar}, and DRO works with ambiguity sets over distributions \citep{wiesemann2014dro}. Nonlinear forecasting via nearest analogs \citep{sugihara1990nonlinear} is a precedent for CAST's retrieval; our heads are learned and causal within a sequence.

\section{The CAST Model}
\label{sec:method}

\paragraph{Problem.} We observe a distribution-valued time series $p_1,\ldots,p_T\in\Delta^{D-1}$ and seek a causal one-step predictor $\hat p_{t+1}=f_\theta(p_{1:t})$; coordinates may be unordered (compositions) or ordered (queue occupancy, AQI, age). CAST parameterizes $f_\theta$ not as an unconstrained map to logits but as a structured simplex transition: history $\to$ empirical successor (retrieval) $\to$ persistence--successor anchor $\to$ bounded local transport. The complete ordered-support transition is
\begin{equation}
\hat p_{t+1}
\;=\;
(1-\rho_t)\bigl[\lambda_t p_t+(1-\lambda_t)r_t\bigr]
\;+\;
\rho_t\, T_t\bigl[\lambda_t p_t+(1-\lambda_t)r_t\bigr],
\label{eq:cast-full}
\end{equation}
with $r_t$ a retrieved empirical successor, $\lambda_t\in[0,1]$ a persistence gate, $T_t$ a radius-bounded row-stochastic transport, and $\rho_t\in[0,\rho_{\max}]$ a bounded transport strength with $\rho_{\max}\le 1$ (so that each stage is a convex combination). For unordered supports the transport stage is disabled ($\rho_t\equiv 0$), yielding the successor-anchor map $\hat p_{t+1}=\lambda_t p_t+(1-\lambda_t)r_t$.

\paragraph{Causal representation and retrieval.} CAST computes causal hidden states $h_{1:T}=\mathrm{Enc}_\theta(p_{1:T})$---a decoder-only Transformer over simplex-aware embeddings---and retrieves from the strictly causal memory $\mathcal{M}_t=\{(h_s,p_{s+1}):s<t\}$ (the strict inequality prevents leakage of $p_{t+1}$). For each retrieval head $m$, queries $q_t^{(m)}=W_q^{(m)}h_t$ and keys $k_s^{(m)}=W_k^{(m)}h_s$ produce causal softmax weights $\alpha_{t,s}^{(m)}\propto\exp\!\bigl((q_t^{(m)})^\top k_s^{(m)}/\sqrt{d_r}\bigr)$ over $s<t$, and each head outputs $r_t^{(m)}=\sum_{s<t}\alpha_{t,s}^{(m)}\,p_{s+1}$. The retrieved successor $r_t=\sum_{m}\eta_{t,m}\,r_t^{(m)}$ with $\eta_t=\mathrm{softmax}(W_\eta h_t)$ is a convex combination of past simplex observations, hence $r_t\in\Delta^{D-1}$; if no valid past successor exists, CAST falls back to persistence $r_t=p_t$.

\paragraph{Persistence--successor anchor.} Pure retrieval is noisy, so CAST forms $a_t=\lambda_t p_t+(1-\lambda_t)r_t$ with gated persistence weight $\lambda_t=\lambda_{\min}+(\lambda_{\max}-\lambda_{\min})\,\sigma(g_\theta(h_t))$. This interpolates between copying the present distribution ($\lambda_t\!\to\!1$) and reusing the retrieved transition ($\lambda_t\!\to\!0$); bounded gating keeps $a_t\in\Delta^{D-1}$ strictly.

\paragraph{Bounded local transport.} For ordered supports, CAST predicts a row-stochastic radius-1 kernel $K_t(j,o)\ge0$, $\sum_{o\in\{-1,0,+1\}}K_t(j,o)=1$, and applies $(T_t a_t)(k)=\sum_j a_t(j)\sum_o K_t(j,o)\,\mathbf 1\{k=\mathrm{clip}(j+o,1,D)\}$. The correction is mixed back into the anchor at strength $\rho_t\in[0,\rho_{\max}]$. To prevent runaway support drift during multi-step rollout, CAST enforces a realized-mean-shift budget: with $\mu(p)=\sum_j j\,p(j)$ and $\Delta\mu_t=\mu(T_t a_t)-\mu(a_t)$, the raw strength is scaled by $g_t=\min\!\bigl(1,\,B_t/(|\Delta\mu_t|+\epsilon)\bigr)$ where $B_t=\delta_\mu+\delta_\sigma\,\sigma_{\mathrm{supp}}(a_t)$. Transport is therefore local in both support radius and realized mean displacement.

\paragraph{Training.} CAST is trained on teacher-forced adjacent pairs $(p_{1:t},p_{t+1})$ with $\mathcal{L}=\mathcal{L}_{\mathrm{step}}+\lambda_{\mathrm{op}}\,\mathcal R_{\mathrm{op}}$, where $\mathcal{L}_{\mathrm{step}}=\tfrac{1}{|\Omega|}\sum_{(b,t)\in\Omega}\mathrm{KL}(p_{b,t+1}\|\hat p_{b,t+1})$ over unmasked positions $\Omega$ and $\mathcal R_{\mathrm{op}}$ is a small target-free operator prior on the transport (penalizing strong transport, off-identity mass, rough kernels across neighboring bins, and realized mean shift). $\mathcal R_{\mathrm{op}}$ depends only on $T_t$ and $a_t$, never on future targets; no rollout loss, multi-horizon target loss, target-moment loss, oracle gate supervision, or validation/test retrieval memory is used. Rollout evaluation iterates the learned one-step map: each prediction is appended to the context and used as input for subsequent predictions. The decomposition of $\mathcal R_{\mathrm{op}}$ is in Appendix~\ref{app:experimental-details}.

\paragraph{Properties.} Four structural properties hold throughout: (i) Causality---$\hat p_{t+1}$ never depends on $p_{t+1}$; (ii) Simplex closure---$r_t,a_t,T_t a_t,\hat p_{t+1}\in\Delta^{D-1}$ by convexity of each stage; (iii) Identity containment---$(\lambda_t,\rho_t)=(1,0)$ recovers persistence, $(0,0)$ recovers pure successor anchoring; (iv) Bounded drift---$W_1(a_t,\hat p_{t+1})\le\rho_t r$ and $|\mu(\hat p_{t+1})-\mu(a_t)|\le\rho_t B_t$, preventing accumulation of support displacement under multi-step rollout. Compactly, CAST = causal successor retrieval + persistence stabilization + bounded local transport: a structured alternative to a sequence-to-distribution Transformer that preserves empirical transition structure for rollout stability while retaining a small, interpretable correction mechanism for ordered distributional motion.

\section{Theoretical Analysis}
\label{sec:theory}

Our theory makes conditional, structural claims. It does not assert that CAST dominates every forecaster on every process; it identifies a data regime, latent transition-kernel aliasing, in which averaging incompatible successors is unavoidable for aliased summaries, while CAST's retrieval--anchor--transport class contains the corresponding Bayes successor. Three pieces are stated below; full proofs and an approximation/retrieval/consistency supplement appear in Appendix~\ref{app:theory-details}.

\paragraph{Setup.} Let $\mathcal F_t=\sigma(p_{1:t})$ be the causal history and $\mathcal G_t^\phi=\sigma(\phi(p_{1:t}))\subseteq\mathcal F_t$ the summary available to a baseline. The hard case is that two or more histories share the same observed summary $\varphi$ but different conditional successors, indexed by a latent regime $Z_t\in\{1,\ldots,K\}$. On $\{\phi(p_{1:t})=\varphi\}$, let $\pi_z=\mathbb P(Z_t=z\mid\varphi)$ and $u_z=\mathbb E[p_{t+1}\mid Z_t=z,\varphi]$. All KL statements assume relative-interior successors or the same small-probability smoothing used in evaluation ($\Delta_\gamma^{D-1}$ when needed). The no-transport anchor class is $\mathcal A_t=\{\lambda p_t+(1-\lambda)r:\lambda\in[0,1],\,r\in\mathcal R_t\}$.

\begin{theorem}[Latent-kernel aliasing lower bound]
\label{thm:aliasing}
Fix $\varphi$ and suppose $\pi_z>0$ for all $z$ and the regime-specific Bayes successors $u_1,\ldots,u_K$ are not all identical. Then every fixed-summary forecaster $f(p_{1:t})=g(\phi(p_{1:t}))$ has conditional KL excess risk (relative to the regime-aware Bayes predictor) at least
\[
\inf_{q\in\Delta^{D-1}}\sum_{z=1}^K\pi_z\,\mathrm{KL}(u_z\,\|\,q)
\;=\;\mathrm{JS}_\pi(u_1,\ldots,u_K)\;>\;0,
\]
where $\mathrm{JS}_\pi$ is the weighted Jensen--Shannon divergence.
\end{theorem}

\noindent\textit{Proof sketch.} Any fixed-summary forecaster must output a single $q=g(\varphi)$ independent of $Z_t$. The standard regime-conditional KL decomposition makes the regime-aware Bayes risk cancel, leaving excess risk $\sum_z\pi_z\,\mathrm{KL}(u_z\|q)$; its minimizer is the mixture $\bar u=\sum_z\pi_z u_z$, and substitution gives $\mathrm{JS}_\pi(u_1,\ldots,u_K)$, which is zero iff all $u_z$ coincide. $\square$

\begin{theorem}[CAST oracle representability]
\label{thm:representability}
Assume that, conditional on $Z_t=z$, the Bayes successor has CAST form $u_z=(1-\rho_z)a_z+\rho_z T_z a_z$, $a_z=\lambda_z p_t+(1-\lambda_z)s_z$, with $s_z\in\Delta^{D-1}$, $\lambda_z\in[0,1]$, $\rho_z\in[0,\rho_{\max}]$ for some $\rho_{\max}\le 1$, and valid local stochastic $T_z$ (for ordered supports). Assume further that (i) there exists a causal representation $h_t=E(p_{1:t})$ identifying $Z_t$; (ii) for each $z$, the causal memory before time $t$ contains an anchor $s_z$, or anchors whose convex hull contains $s_z$; (iii) CAST gates and transport head can realize $(\lambda_z,\rho_z,T_z)$. Then there is a CAST parameterization (or a limit of finite-softmax parameterizations) such that $\hat p_{t+1}=u_{Z_t}$; the conditional KL excess risk of this oracle CAST predictor is zero, and vanishing in deterministic benchmarks.
\end{theorem}

\noindent\textit{Proof sketch.} Because $h_t$ identifies $Z_t$, retrieval heads can concentrate their mass (exactly, or as a finite-softmax limit) on the memory anchor $s_{Z_t}$; the anchor and transport gates then realize $(\lambda_{Z_t},\rho_{Z_t},T_{Z_t})$; substitution into \eqref{eq:cast-full} recovers $u_{Z_t}$. $\square$

\begin{theorem}[Ordered-support transport separation]
\label{thm:transport}
Assume the support is ordered, $Z_t$ is identified, and the no-transport anchor class $\mathcal A_z^0=\{\lambda p_t+(1-\lambda)r:\lambda\in[0,1],\,r\in\mathcal R_z^0\}$ has $L^1$ gap $\delta_z=\inf_{q\in\mathcal A_z^0}\|u_z-q\|_1>0$ for each $z$. Then even an oracle no-transport method that knows $Z_t$ incurs conditional KL excess risk
\[
\inf_{q_z\in\mathcal A_z^0}\sum_z\pi_z\,\mathrm{KL}(u_z\,\|\,q_z)\;\ge\;\tfrac12\sum_z\pi_z\delta_z^2,
\]
while CAST attains zero oracle excess risk whenever $u_z$ lies in the CAST local-transport class (Theorem~\ref{thm:representability}).
\end{theorem}

\noindent\textit{Proof sketch.} For each $z$, $\|u_z-q_z\|_1\ge\delta_z$ by construction; Pinsker's inequality gives $\mathrm{KL}(u_z\|q_z)\ge\frac12\delta_z^2$; averaging over $z$ yields the bound. The boundary case $\delta_z=0$ (successor already in the hull) is exactly when the theorem gives no separation, matching the intended reading that transport is only a formal advantage when local motion is not already representable as a no-transport mixture. $\square$

\paragraph{Supporting results and scope.} Appendix~\ref{app:theory-details} supplements these theorems with: an approximation bound $\|\tilde u-u\|_1\le \epsilon_r+2\epsilon_\lambda+\rho_{\max}\epsilon_T+2\epsilon_\rho$ for imperfect CAST estimates (with a $1/\gamma$-Lipschitz KL counterpart on $\Delta_\gamma^{D-1}$); a retrieval-consistency statement that the nearest-memory retrieval error at distance $d_t$ is at most $Ld_t+\|\xi\|_1$ under an $L$-Lipschitz successor map, so a dense causal memory makes the anchor converge to the Bayes successor; simplex-closure and $W_1$ drift control used in Section~\ref{sec:method}; and a queueing corollary specializing to $u_{\mathrm{up/down}}=(1-\rho)p^\star+\rho T_{\mathrm{right/left}}p^\star$, where a phase-blind baseline pays $\mathrm{JS}_{1/2}(u_{\mathrm{up}},u_{\mathrm{down}})$ and no-transport anchor methods pay an additional $\tfrac12\delta^2$---tested below on a controlled synthetic.

\section{Experimental Setup}
\label{sec:experiments}

\paragraph{Task and evaluation modes.} Each method is trained to predict $\hat p_{t+1}=f(p_{1:t})$ and evaluated under the same online information structure. We report two modes. Offline one-step: teacher-forced prediction on held-out chronological target positions; this isolates the quality of the learned one-step conditional map. Autoregressive rollout: from an observed context the model repeatedly feeds its own prediction back into its history, measuring long-horizon drift and stability of the learned transition operator.

\paragraph{Benchmark suite.} The suite contains eleven distribution-valued sections: eight public composition panels---BioTIME (ecological composition) \citep{dornelas2025biotime}, Ember Monthly (electricity generation mix) \citep{ember2026monthlyelectricity}, OWID Dietary (diet composition) \citep{owid2023dietcompositions}, CDC Weekly Deaths (cause-of-death share) \citep{cdc_nchs2023weeklydeaths}, BLS QCEW (industry employment share) \citep{bls2026qcew}, EPA AirData AQI (AQI category share) \citep{epa2026airdata}, NOAA Storm Events (event-type share) \citep{noaa_ncei_stormevents}, NYC TLC Trip (pickup-zone share) \citep{nyctlc_trip_records}---and three queue-occupancy sections (Homogeneous, Nonhomogeneous, Combined) built from $G/G/1$ and $G_t/G/1$ simulations following the generation protocol of \citet{di2025transformerqueue}. The suite is intentionally heterogeneous: unordered composition panels test whether a method can exploit causal successor information without coordinate order, while ordered datasets (EPA AQI and all queue sections) additionally test whether local mass motion helps. Appendix~\ref{app:benchmarks} gives full data sources, licenses, preprocessing, and per-section split statistics; Table~\ref{tab:main-dataset-summary-short} summarizes the sections.

\begin{table}[t]
\centering
\caption{Benchmark suite summary. Ordered = the support has a natural order relevant to local transport ($\rho_t\not\equiv 0$). Queue sections follow the $G/G/1$ and $G_t/G/1$ generation protocol of \citet{di2025transformerqueue}.}
\label{tab:main-dataset-summary-short}
\scriptsize
\setlength{\tabcolsep}{3.0pt}
\begin{tabular}{@{}lllrrrc@{}}
\toprule
Section & Type & Sequence unit & \#Seq. & $D$ & Length (min / med / max) & Ordered \\
\midrule
BioTIME & ecological composition & study-site annual & 646 & 128 & 15 / 25 / 57 & no \\
Ember Monthly & energy generation share & country monthly & 87 & 9 & 40 / 109 / 327 & no \\
OWID Dietary & diet composition & country annual & 135 & 25 & 63 / 63 / 63 & no \\
CDC Weekly Deaths & mortality cause share & state weekly & 50 & 6 & 190 / 190 / 190 & no \\
BLS QCEW & industry employment share & state annual & 53 & 13 & 15 / 15 / 15 & no \\
EPA AirData AQI & AQI category share & county monthly & 800 & 6 & 60 / 191 / 300 & \textbf{yes} \\
NOAA Storm Events & event-type share & state monthly & 49 & 32 & 61 / 244 / 300 & no \\
NYC TLC Trip & pickup-zone share & pickup-hour weekly & 24 & 266 & 52 / 53 / 53 & no \\
Queue Homogeneous & $G/G/1$ occupancy & queue trajectory & 10{,}000 & 596 & 101 / 549 / 16621 & \textbf{yes} \\
Queue Nonhomogeneous & $G_t/G/1$ occupancy & queue trajectory & 9{,}900 & 562 & 27 / 547 / 1758 & \textbf{yes} \\
Queue Combined & pooled $G/G/1 \cup G_t/G/1$ & queue trajectory & 19{,}900 & 596 & 27 / 548 / 16621 & \textbf{yes} \\
\bottomrule
\end{tabular}
\end{table}

\paragraph{Baselines.} We compare CAST with 15 baselines that all produce simplex-valued forecasts under the same loss masks. Statistical/compositional: \texttt{persistence} ($\hat p_{t+1}=p_t$), \texttt{analog\_successor} (train-split memory bank of context-successor windows with weighted nearest-neighbor retrieval), \texttt{ilr\_var} (VAR in isometric log-ratio coordinates), \texttt{compositional\_ets} (ilr-space ETS/state-space). Neural: \texttt{transformer\_naive} (decoder-only Transformer with simplex head), \texttt{gru}, \texttt{lstm}, \texttt{tcn}, \texttt{nhits} \citep{challu2023nhits}, \texttt{dlinear} \citep{zeng2023transformers}, \texttt{timemixer} \citep{wang2024timemixer}, \texttt{itransformer} \citep{liu2024itransformer}, \texttt{informer} \citep{zhou2021informer}, \texttt{autoformer} \citep{wu2021autoformer}, and \texttt{tide} \citep{das2023tide}. Each neural backbone is adapted to distribution-valued inputs and capped with a final simplex-valued head; all trainable baselines use one-step KL loss on the same masked adjacent pairs as CAST.

\paragraph{Splits and protocol.} All splits are chronological. For the eight public panels, each split file stores the same probability-vector sequences with disjoint loss masks over scored one-step targets; because the models are causal, held-out target scoring is teacher-forced online forecasting, not a random shuffle. Per-dataset split policies and target counts are in Appendix~\ref{app:experimental-details}. The queue benchmark uses a file-level 70/10/20 split by whole queue system, stratified by support-width deciles, so queue evaluation measures generalization to unseen simulated systems rather than interpolation between overlapping windows from seen systems. Rollout contexts/horizons range from 8/2 (BLS QCEW) to 128/64 (queue sections); see Appendix~\ref{app:experimental-details}.

\paragraph{Training.} All trainable models optimize one-step KL with AdamW, learning rate $3\!\times\!10^{-4}$, 200-step warmup, weight decay 0.1, gradient clipping 1.0, 10k iterations (12k for queues), hidden dimension 256, 6 layers, 4 heads, dropout 0.1, batch size 16 (4 for queues). Models are selected by the lowest validation one-step KL loss. No rollout loss, multi-horizon target loss, target-moment loss, oracle gate supervision, or validation/test retrieval memory is used for CAST. Fit-only baselines (\texttt{persistence}, \texttt{analog\_successor}, \texttt{ilr\_var}, \texttt{compositional\_ets}) are fitted on the training split without SGD. Full hyperparameters, dataset-specific block sizes, CAST configuration, and compute (A100/H200/L40S two-GPU nodes, 16 cores, $\ge$128\,GB host memory) are in Appendix~\ref{app:experimental-details}.

\paragraph{Metrics and reporting.} Lower is better for all metrics. We report one-step KL, JSD, $L^1$, and ordered-support $W_1$; the primary one-step metric is offline KL and the primary long-horizon metric is rollout JSD. The main aggregate comparison is average rank over the 11 sections; top-1 wins are reported as an auxiliary summary because ranks retain how close a method remains when it is not the winner. The main tables report seed 42 as a broad common-protocol comparison across 16 methods and 11 datasets; we additionally run a targeted five-seed CAST robustness study on the public non-queue sections (Appendix~\ref{app:cast-seed-robustness}).

\section{Results}
\label{sec:results}

\begin{table}[t]
\centering
\caption{Average ranks across the 11 benchmark sections (lower is better). All 16 methods are complete on all 11 sections; ranks are directly comparable. CAST is the only method with sub-2 average rank on either metric.}
\label{tab:main-average-ranks}
\scriptsize
\setlength{\tabcolsep}{3.2pt}
\begin{tabular}{@{}l rrrr@{\hspace{1.0em}}l rrrr@{}}
\toprule
Method & KL$_\text{avg}$ & RO$_\text{avg}$ & Mean & Best & Method & KL$_\text{avg}$ & RO$_\text{avg}$ & Mean & Best \\
\midrule
\textbf{CAST}       & \textbf{1.27} & \textbf{1.91} & \textbf{1.59} & 1 / 1 & iTransformer & 8.73  & 9.73  & 9.23  & 5 / 7 \\
Comp.\ ETS          & 5.27          & 3.91          & 4.59          & 2 / 1 & TCN          & 8.64  & 11.27 & 9.95  & 3 / 7 \\
Persistence         & 5.64          & 3.73          & 4.68          & 2 / 1 & TimeMixer    & 9.00  & 11.36 & 10.18 & 3 / 5 \\
N-HiTS              & 5.27          & 5.55          & 5.41          & 1 / 3 & TiDE         & 10.82 & 10.18 & 10.50 & 8 / 7 \\
Informer            & 6.45          & 7.18          & 6.82          & 4 / 3 & LSTM         & 12.27 & 9.45  & 10.86 & 7 / 2 \\
Transformer         & 9.00          & 6.64          & 7.82          & 3 / 2 & Analog succ.\ & 12.00 & 11.18 & 11.59 & 5 / 4 \\
ilr-VAR             & 7.45          & 8.64          & 8.05          & 1 / 2 & Autoformer   & 12.55 & 13.64 & 13.09 & 6 / 6 \\
GRU                 & 8.55          & 8.27          & 8.41          & 3 / 3 & DLinear      & 13.09 & 13.36 & 13.23 & 6 / 11 \\
\bottomrule
\end{tabular}
\end{table}

\paragraph{Aggregate ranks.} Table~\ref{tab:main-average-ranks} reports average rank across 11 sections. CAST is first on both offline KL (1.27) and rollout JSD (1.91), with mean rank 1.59 vs.\ 4.59 for the closest non-CAST competitor (Comp.\ ETS) and 5.41 for the closest neural backbone (N-HiTS). CAST is top-1 on 8/11 sections for both metrics (Table~\ref{tab:main-dataset-summary}) and top-2 on all 11 for offline KL. The three non-winning offline sections (EPA AirData AQI, NOAA Storm Events, Queue Homogeneous) are all sections where CAST wins rollout JSD, so one-step losses do not propagate to long-horizon stability. Strong classical baselines (persistence, Comp.\ ETS) are rollout-competitive on smooth panels but collapse on Queue Nonhomogeneous/Combined---the regime predicted by Theorems~\ref{thm:aliasing}--\ref{thm:transport}.

\begin{table}[t]
\centering
\caption{Dataset-level comparison: CAST vs.\ strongest non-CAST baseline on each metric (lower is better; bold = winner). CAST is top-1 on 8/11 sections for both metrics and top-2 on all 11 sections for offline KL.}
\label{tab:main-dataset-summary}
\scriptsize
\setlength{\tabcolsep}{2.4pt}
\resizebox{\textwidth}{!}{%
\begin{tabular}{@{}l rl rl@{}}
\toprule
& \multicolumn{2}{c}{Offline KL $\downarrow$} & \multicolumn{2}{c}{Rollout JSD $\downarrow$} \\
\cmidrule(lr){2-3}\cmidrule(l){4-5}
Dataset & CAST vs.\ strongest non-CAST & Winner & CAST vs.\ strongest non-CAST & Winner \\
\midrule
BioTIME              & \textbf{0.1313} \,/\, 0.1844 N-HiTS        & \textbf{CAST}  & 0.03232 \,/\, \textbf{0.02206} Comp.\ ETS   & Comp.\ ETS \\
Ember Monthly        & \textbf{0.04342} \,/\, 0.05583 Persist.    & \textbf{CAST}  & \textbf{0.007129} \,/\, 0.01146 ilr-VAR     & \textbf{CAST} \\
OWID Dietary         & \textbf{0.009096} \,/\, 0.009747 Comp.\ ETS & \textbf{CAST}  & 0.005478 \,/\, \textbf{0.005441} Persist.    & Persist.\ \\
CDC Weekly Deaths    & \textbf{0.01079} \,/\, 0.01423 N-HiTS      & \textbf{CAST}  & 0.01743 \,/\, \textbf{0.003676} Persist.    & Persist.\ \\
BLS QCEW             & \textbf{0.0008773} \,/\, 0.001242 Persist. & \textbf{CAST}  & \textbf{0.0002459} \,/\, 0.0002461 Persist. & \textbf{CAST} \\
EPA AirData AQI      & 0.1184 \,/\, \textbf{0.09518} N-HiTS       & N-HiTS          & \textbf{0.03749} \,/\, 0.04465 Transf.\      & \textbf{CAST} \\
NOAA Storm Events    & 1.004 \,/\, \textbf{0.805} N-HiTS          & N-HiTS          & \textbf{0.2042} \,/\, 0.2161 LSTM            & \textbf{CAST} \\
NYC TLC Trip         & \textbf{0.02168} \,/\, 0.02742 N-HiTS       & \textbf{CAST}  & \textbf{0.009625} \,/\, 0.01169 Comp.\ ETS   & \textbf{CAST} \\
Queue Homogeneous    & 0.01576 \,/\, \textbf{0.01466} ilr-VAR     & ilr-VAR         & \textbf{0.07323} \,/\, 0.0887 Persist.\      & \textbf{CAST} \\
Queue Nonhomogeneous & \textbf{0.07101} \,/\, 0.07959 Persist.     & \textbf{CAST}  & \textbf{0.03857} \,/\, 0.05524 Persist.\     & \textbf{CAST} \\
Queue Combined       & \textbf{0.007519} \,/\, 0.0103 Persist.     & \textbf{CAST}  & \textbf{0.081} \,/\, 0.08827 Comp.\ ETS      & \textbf{CAST} \\
\midrule
\textbf{CAST top-1} & \multicolumn{2}{c}{\textbf{8/11 sections}} & \multicolumn{2}{c}{\textbf{8/11 sections}} \\
\bottomrule
\end{tabular}%
}
\end{table}

\begin{figure}[t]
    \centering
    \includegraphics[width=0.95\linewidth]{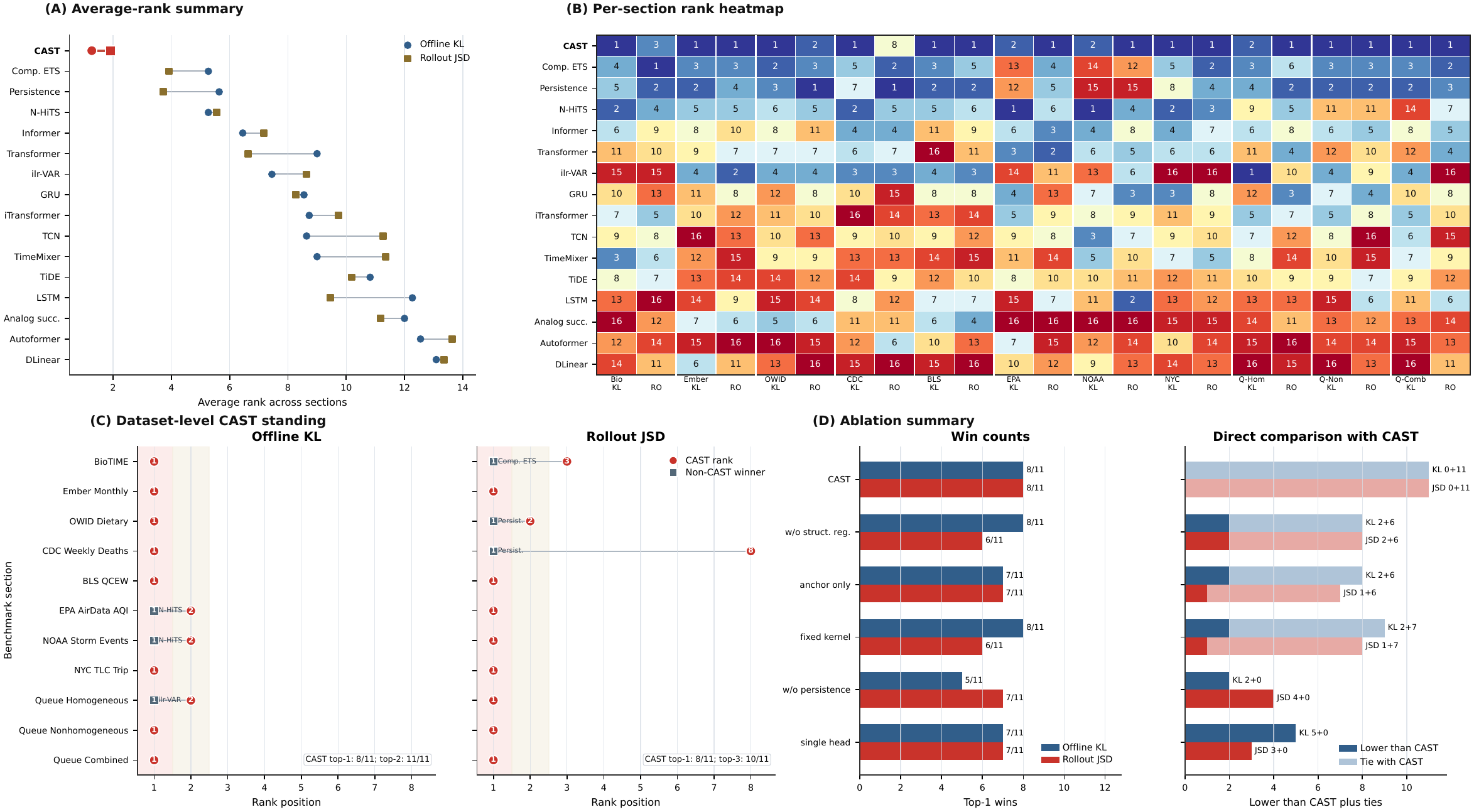}
    \caption{\textbf{Main empirical results and component ablations.}
    \textbf{(A)} Average offline KL and rollout JSD ranks across 11 sections for all 16 methods (lower is better).
    \textbf{(B)} Per-dataset rank heatmap for offline KL and rollout JSD; cooler = stronger rank.
    \textbf{(C)} Dataset-level standing vs.\ the best non-CAST baseline: red circles = CAST's rank, gray squares = non-CAST winner when CAST is not top, horizontal segments = the rank gap.
    \textbf{(D)} CAST component study: top-1 wins against the best non-CAST baseline and direct section-level comparisons with CAST.}
    \label{fig:main-results}
\end{figure}

The main sweep is single-seed by design; a five-seed CAST robustness study on the non-queue sections (Appendix~\ref{app:cast-seed-robustness}) keeps the CAST interval separated from the immediately worse baseline in almost all CAST-winning cases (the only unstable exception is OWID Dietary, where leading methods have very small absolute gaps).

\subsection{Component ablation}
\label{sec:results-ablation}

\begin{table}[t]
\centering
\caption{Component and capacity-control summary. Win counts are each CAST variant vs.\ the strongest non-CAST baseline per section (so 8/11 matches full CAST). Median $\Delta$ is variant/CAST$-1$; positive means higher error.}
\label{tab:ablation-summary}
\scriptsize
\setlength{\tabcolsep}{3.0pt}
\begin{tabular}{@{}llrrrrrr@{}}
\toprule
Variant & Study & KL wins & RO wins & Med.\ KL $\Delta$ & Med.\ RO $\Delta$ & KL$<$CAST & RO$<$CAST \\
\midrule
\textbf{CAST} & Full model & \textbf{8/11} & \textbf{8/11} & -- & -- & -- & -- \\
w/o structural reg. & Strict ablation & 8/11 & 6/11 & +0.0\% & +0.0\% & 2/11 & 2/11 \\
anchor only & Strict ablation & 7/11 & 7/11 & +0.0\% & +0.0\% & 2/11 & 1/11 \\
single-head retrieval & Capacity control & 7/11 & 7/11 & +0.5\% & +2.4\% & 5/11 & 3/11 \\
fixed local kernel & Strict ablation & 8/11 & 6/11 & +0.0\% & +0.0\% & 2/11 & 1/11 \\
w/o persistence mix & Strict ablation & 5/11 & 7/11 & +14.5\% & +0.8\% & 2/11 & 4/11 \\
\bottomrule
\end{tabular}
\end{table}

Table~\ref{tab:ablation-summary} (visualized in Figure~\ref{fig:main-results}D) compares each variant against the strongest non-CAST baseline per section. Removing structural regularization keeps offline KL wins (8/11) but drops rollout wins from 8/11 to 6/11---the target-free operator prior is a rollout-stability mechanism. Removing the persistence mix ($\lambda_t\!\equiv\!0$) is the most damaging change: KL wins drop from 8/11 to 5/11 ($+14.5\%$ median $\Delta$), with Ember Monthly and Queue Combined rising by orders of magnitude. A fixed local kernel keeps offline KL wins but degrades rollout where ordered transport matters (Queue Nonhomogeneous rollout JSD $+223\%$). Single-head retrieval is occasionally stronger than CAST on individual sections but does not improve the 11-section summary. Full per-section numbers are in Appendix~\ref{app:ablation-full}.

\subsection{Synthetic validation of the aliasing theory}
\label{sec:results-synthetic}

\begin{table}[t]
\centering
\caption{Synthetic aliasing experiment ($K=2$ equal-probability regimes, identical $p^\star$, opposite radius-1 local transports). Fixed-summary forecasters must collapse to the JS-average; phase-aware anchor-only cannot leave the no-transport hull; CAST closes the gap.}
\label{tab:synthetic-aliasing-experiment}
\small
\setlength{\tabcolsep}{3.2pt}
\resizebox{\linewidth}{!}{%
\begin{tabular}{@{}lccc@{}}
\toprule
Method & KL $\downarrow$ & JSD $\downarrow$ & $L^1$ $\downarrow$ \\
\midrule
Fixed-summary optimum (Bayes mixture average) & 0.044893 & 0.011562 & 0.238730 \\
Trained current-only neural forecaster & $0.044899\pm 3.3\mathrm{e}{-6}$ & $0.011564\pm 3.2\mathrm{e}{-6}$ & $0.238774\pm 2.9\mathrm{e}{-5}$ \\
Phase-aware anchor only (no transport) & 0.047840 & 0.011722 & 0.246619 \\
CAST oracle local transport & \textbf{0} & \textbf{0} & \textbf{0} \\
Trained CAST local-transport head & $\mathbf{1.7\mathrm{e}{-8}\pm 5.0\mathrm{e}{-8}}$ & $\mathbf{0\pm 6.1\mathrm{e}{-8}}$ & $\mathbf{2.3\mathrm{e}{-4}\pm 1.7\mathrm{e}{-5}}$ \\
\bottomrule
\end{tabular}%
}
\end{table}

We implement the exact setup of the queueing corollary (Section~\ref{sec:theory}): two equal-probability latent phases, identical current simplex state $p^\star$, and opposite radius-1 local transports $u_{\mathrm{up}}=(1-\rho)p^\star+\rho T_{\mathrm{right}}p^\star$, $u_{\mathrm{down}}=(1-\rho)p^\star+\rho T_{\mathrm{left}}p^\star$. Table~\ref{tab:synthetic-aliasing-experiment} shows the predicted gap cleanly: a fixed-summary forecaster and a trained current-only neural forecaster both sit almost exactly on the Bayes mixture average (KL $\approx 0.04489$), the phase-aware anchor-only comparator is worse than the fixed-summary optimum because it cannot leave the no-transport hull (KL $0.04784$), and both the CAST oracle and a trained CAST local-transport head close the gap to essentially zero (trained CAST KL $\approx 1.7\!\times\!10^{-8}$). This is not a parameter-tuning artefact: the fixed-summary optimum is provably $\mathrm{JS}_{1/2}(u_{\mathrm{up}},u_{\mathrm{down}})$ by Theorem~\ref{thm:aliasing}, the anchor-only gap is provably $\tfrac12\delta^2$ by Theorem~\ref{thm:transport}, and the CAST success is Theorem~\ref{thm:representability}.

\subsection{Qualitative visualization analysis}
\label{sec:results-visualization}

\begin{figure}[t]
    \centering
    \includegraphics[width=0.88\linewidth]{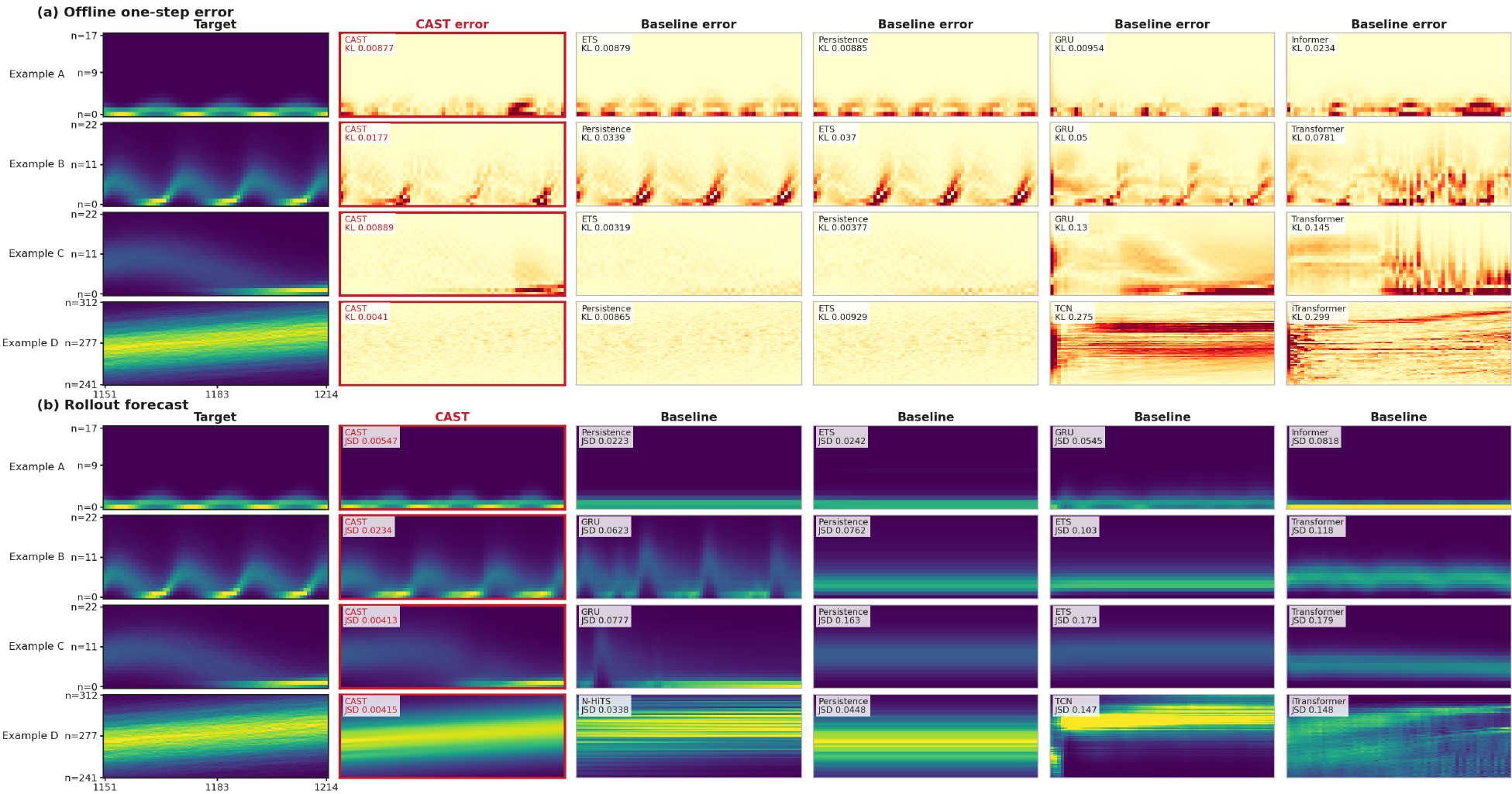}
    \caption{\textbf{Qualitative offline and rollout on held-out queueing systems.}
    \textbf{(a)} Offline one-step residual heat maps (annotated by mean KL): CAST's errors are diffuse and low-magnitude; baselines show structured errors aligned with target ridges.
    \textbf{(b)} Autoregressive rollout (annotated by mean JSD): CAST preserves the moving occupancy mass and changing support shape; persistence and Comp.\ ETS drift toward static bands, N-HiTS over-smooths, and Transformer produces off-support artefacts. CAST's rollout advantage is driven by preserving ordered local mass motion under its own past predictions---the mechanism the bounded transport + shift-budget design enables.}
    \label{fig:main-offline-rollout-examples}
\end{figure}

Figure~\ref{fig:main-offline-rollout-examples} complements the aggregate numbers on held-out queues. In offline residuals (panel a) CAST's error is visually structureless (small amplitude, no ridge alignment) while persistence/Comp.\ ETS show diagonal stripes (they lag moving mass) and N-HiTS/Transformer show patchy high-error regions (misplaced modes). Autoregressive rollouts (panel b) turn these one-step differences into qualitative divergence: persistence freezes into a static band, Comp.\ ETS and N-HiTS smear the mass into a plateau, Transformer develops off-support ghost modes, and CAST continues to advance the occupancy mode and track the narrowing of the distribution---consistent with a bounded-transport operator acting on a persistence--retrieval anchor. CAST's advantage is thus not a small numerical edge on one-step KL but a structural advantage in how predicted probability mass evolves under the model's own past forecasts.

\paragraph{Where aliasing matters most.} The nearest-neighbor diagnostic of Appendix~\ref{app:theory-diagnostics} isolates strong aliasing in EPA AirData AQI, NOAA Storm Events, and Queue Nonhomogeneous/Combined---precisely where CAST's rollout advantage is largest (rank 1 on all four, $16$--$30\%$ rollout-JSD gaps on the two queue sections). Queue Homogeneous is the opposite case: CAST's advantage there reflects rollout stability and ordered local transport rather than aliasing disambiguation (loses offline KL to ilr-VAR but wins rollout JSD by $17\%$).

\section{Conclusion}
\label{sec:conclusion}

CAST structures the transition operator for distribution-valued time series as a causal successor-retrieval + persistence-anchor + bounded-local-transport composition, keeping every stage inside $\Delta^{D-1}$. The theory gives a weighted Jensen--Shannon excess-risk lower bound for fixed-summary forecasters under latent-kernel aliasing and places the regime-aware Bayes successor inside CAST's hypothesis class, with an added Pinsker separation for ordered supports. CAST attains the best average rank on both one-step KL and rollout JSD (8/11 wins each) across eleven benchmarks; ablations, seed robustness, a synthetic aliasing experiment, and queueing rollout visualizations show that its advantage comes from preserving ordered local mass motion under its own past predictions.

\section{Limitations}
\label{sec:limitations}

CAST targets recurring transition regimes with optional ordered mass motion; fully unordered supports with purely smooth dynamics may favor a different model. The main sweep is single-seed (complemented by a five-seed CAST robustness check in Appendix~\ref{app:cast-seed-robustness}), and the aliasing theory is an existence argument, not a finite-sample training guarantee. Training uses a one-step KL objective with rollout as a pure evaluation protocol; long-horizon training, longer-range retrieval, multi-radius transport, and integration with risk-sensitive decision making are natural extensions.

{
\small
\bibliographystyle{unsrtnat}
\bibliography{references}
}

\newpage
\appendix
\section{Theory: Full Proofs and Supporting Results}
\label{app:theory-details}

We give the full proofs of the three main theorems stated in Section~\ref{sec:theory}, together with approximation, retrieval-consistency, and simplex-closure supplements that are summarized in the main text.

\subsection{Setup recap}
\label{app:theory-setup}

Let $p_t\in\Delta^{D-1}=\{p\in\mathbb R_{\ge 0}^D:\sum_j p(j)=1\}$; coordinates may be unordered categorical or ordered support states. Let $\mathcal F_t=\sigma(p_{1:t})$ be the observed causal history and $\mathcal G_t^\phi=\sigma(\phi(p_{1:t}))\subseteq\mathcal F_t$ the summary available to a baseline. The central hard case is latent-kernel aliasing: two or more histories have the same observed summary but different conditional successors, indexed by $Z_t\in\{1,\ldots,K\}$. For an aliasing event $\phi(p_{1:t})=\varphi$, let $\pi_z=\mathbb P(Z_t\!=\!z\mid\varphi)$ and $u_z=\mathbb E[p_{t+1}\mid Z_t\!=\!z,\varphi]$. In stochastic settings the total expected KL risk has an irreducible within-regime term, so the aliasing lower bound is stated as an excess-risk term over the regime-aware Bayes predictor; in deterministic one-step benchmarks that term is zero. For clean KL statements we either assume interior-simplex successors or apply the same small-probability smoothing used in the empirical KL evaluation; when needed we write $\Delta_\gamma^{D-1}=\{p:p(j)\ge\gamma\}$.

The CAST hypothesis class, for ordered supports with radius-$r$ transport, is $\hat p=(1-\rho)a+\rho T a$ with $a=\lambda p+(1-\lambda)r$, $r\in\mathcal R$ retrieved from the causal memory, $\lambda\in[0,1]$, $\rho\in[0,\rho_{\max}]$ for some $\rho_{\max}\le 1$, and row-stochastic local $T$ (identity for unordered supports, where we set $\rho=0$). The no-transport anchor class is $\mathcal A_t=\{\lambda p_t+(1-\lambda)r:\lambda\in[0,1],\,r\in\mathcal R_t\}$; an ordered-transport separation is only meaningful when the no-transport hull does not already contain the transported successor.

\subsection{Theorem~\ref{thm:aliasing} (aliasing lower bound): restatement and proof}

\begin{restatement}[Theorem~\ref{thm:aliasing}: latent-kernel aliasing lower bound]
Fix an aliasing summary value $\varphi$. Suppose $\mathbb P(Z_t=z\mid\phi(p_{1:t})=\varphi)=\pi_z>0$ with $\sum_z\pi_z=1$ and the regime-specific Bayes successors $u_1,\ldots,u_K$ are not all identical. Then every fixed-summary forecaster $f(p_{1:t})=g(\phi(p_{1:t}))$ has conditional KL excess risk, relative to the regime-aware Bayes predictor, at least
\[
\inf_{q\in\Delta^{D-1}}\sum_{z=1}^K\pi_z\,\mathrm{KL}(u_z\,\|\,q)=\mathrm{JS}_\pi(u_1,\ldots,u_K)>0,
\]
where $\mathrm{JS}_\pi$ is the weighted Jensen--Shannon divergence.
\end{restatement}

\begin{proof}
On the event $\phi(p_{1:t})=\varphi$, any fixed-summary forecaster $f=g\circ\phi$ outputs a single distribution $q=g(\varphi)$ independent of $Z_t$. For stochastic targets, the regime-conditional expected KL decomposes as
\[
\mathbb E\!\left[\mathrm{KL}(p_{t+1}\|q)\mid Z_t\!=\!z,\varphi\right]
=\mathbb E\!\left[\mathrm{KL}(p_{t+1}\|u_z)\mid Z_t\!=\!z,\varphi\right]+\mathrm{KL}(u_z\|q),
\]
where the first term is the regime-aware Bayes risk and does not depend on $q$. The fixed-summary excess risk is therefore
\[
R(q)=\sum_z\pi_z\,\mathrm{KL}(u_z\|q).
\]
The minimizer of $R(q)$ over $q\in\Delta^{D-1}$ is the mixture $\bar u=\sum_z\pi_z u_z$; substituting $\bar u$ into $R$ gives
\[
R(\bar u)=\sum_z\pi_z\,\mathrm{KL}(u_z\,\|\,\bar u)=\mathrm{JS}_\pi(u_1,\ldots,u_K),
\]
the weighted Jensen--Shannon divergence. $\mathrm{JS}_\pi$ is zero iff all $u_z$ are identical, so under the theorem's assumption it is strictly positive.
\end{proof}

\subsection{Theorem~\ref{thm:representability} (oracle representability): restatement and proof}

\begin{restatement}[Theorem~\ref{thm:representability}: CAST oracle representability]
Suppose that, conditional on $Z_t=z$, the Bayes successor has CAST form $u_z=(1-\rho_z)a_z+\rho_z T_z a_z$, $a_z=\lambda_z p_t+(1-\lambda_z)s_z$, with $s_z\in\Delta^{D-1}$, $\lambda_z\in[0,1]$, $\rho_z\in[0,\rho_{\max}]$ for some $\rho_{\max}\le 1$, and valid local stochastic $T_z$ (for ordered supports). Assume further that (i) there exists a causal representation $h_t=E(p_{1:t})$ identifying $Z_t$; (ii) for each $z$, the causal memory before time $t$ contains an anchor $s_z$, or anchors whose convex hull contains $s_z$; (iii) CAST gates and transport head can realize $(\lambda_z,\rho_z,T_z)$. Then there exists a CAST parameterization, or a limit of finite-softmax parameterizations, such that $\hat p_{t+1}=u_{Z_t}$; the conditional KL excess risk of this oracle CAST predictor is zero (vanishing in deterministic benchmarks).
\end{restatement}

\begin{proof}
Because the causal representation $h_t=E(p_{1:t})$ identifies $Z_t$, the retrieval heads can place all retrieval mass---or arbitrarily concentrated finite-softmax mass---on the memory anchor whose value equals $s_{Z_t}$ (by assumption (ii), such an anchor or convex hull of anchors exists in $\mathcal M_t$). The anchor gate sets $a_t=\lambda_{Z_t}p_t+(1-\lambda_{Z_t})s_{Z_t}$, and the transport head and strength gate set $T_t=T_{Z_t}$, $\rho_t=\rho_{Z_t}$. Substituting these choices into $\hat p_{t+1}=(1-\rho_t)a_t+\rho_t T_t a_t$ gives $\hat p_{t+1}=(1-\rho_{Z_t})a_{Z_t}+\rho_{Z_t} T_{Z_t}a_{Z_t}=u_{Z_t}$, which is the regime-aware Bayes successor. On $\Delta_\gamma^{D-1}$ (or under the smoothing convention), the conditional KL excess risk of this oracle CAST predictor is zero; on deterministic benchmarks it is zero one-step KL to the target.
\end{proof}

\subsection{Theorem~\ref{thm:transport} (ordered-support transport separation): restatement and proof}

\begin{restatement}[Theorem~\ref{thm:transport}: ordered-support transport separation]
Assume the support is ordered, $Z_t$ is identified, and the no-transport anchor class $\mathcal A_z^0=\{\lambda p_t+(1-\lambda)r:\lambda\in[0,1],\,r\in\mathcal R_z^0\}$ has $L^1$ gap $\delta_z=\inf_{q\in\mathcal A_z^0}\|u_z-q\|_1>0$ for each $z$. Then even an oracle no-transport method that knows $Z_t$ incurs conditional KL excess risk
\[
\inf_{q_z\in\mathcal A_z^0}\sum_z\pi_z\,\mathrm{KL}(u_z\,\|\,q_z)\ge\tfrac12\sum_z\pi_z\delta_z^2,
\]
while CAST attains zero oracle excess risk whenever $u_z$ lies in the CAST local-transport class.
\end{restatement}

\begin{proof}
By construction, for each $z$, every $q_z\in\mathcal A_z^0$ satisfies $\|u_z-q_z\|_1\ge\delta_z$. Pinsker's inequality (with KL measured in natural logarithms, matching the experimental evaluation) gives
\[
\mathrm{KL}(u_z\|q_z)\ge\tfrac12\|u_z-q_z\|_1^2\ge\tfrac12\delta_z^2.
\]
Averaging over $z$ with weights $\pi_z$ yields $\inf_{q_z\in\mathcal A_z^0}\sum_z\pi_z\mathrm{KL}(u_z\|q_z)\ge\tfrac12\sum_z\pi_z\delta_z^2$. The boundary case $\delta_z=0$ (transported successor already in the no-transport hull) gives no separation and matches the intended reading: transport is a formal advantage only when ordered local motion is not already representable as a no-transport mixture.
\end{proof}

\subsection{Proposition on approximation under imperfect retrieval/transport}

\begin{proposition}[Approximation under imperfect CAST estimates]
\label{prop:approx-app}
Let the oracle CAST successor be $u=(1-\rho)a+\rho Ta$, $a=\lambda p+(1-\lambda)r$, and let an estimated CAST predictor use $\tilde r,\tilde\lambda,\tilde\rho,\tilde T$, producing $\tilde u=(1-\tilde\rho)\tilde a+\tilde\rho\tilde T\tilde a$. Assume $\rho,\tilde\rho\in[0,\rho_{\max}]$ with $\rho_{\max}\le 1$, $\|\tilde r-r\|_1\le\epsilon_r$, $|\tilde\lambda-\lambda|\le\epsilon_\lambda$, $|\tilde\rho-\rho|\le\epsilon_\rho$, and $\|\tilde T\tilde a-T\tilde a\|_1\le\epsilon_T$. Then
\[
\|\tilde u-u\|_1\;\le\;\epsilon_r+2\epsilon_\lambda+\rho_{\max}\epsilon_T+2\epsilon_\rho.
\]
If additionally $u,\tilde u\in\Delta_\gamma^{D-1}$, then $\mathrm{KL}(u\,\|\,\tilde u)\le\tfrac{1}{\gamma}\bigl[\epsilon_r+2\epsilon_\lambda+\rho_{\max}\epsilon_T+2\epsilon_\rho\bigr]$.
\end{proposition}
\textit{Proof sketch.} The anchor error obeys $\|\tilde a-a\|_1\le\|(1-\tilde\lambda)(\tilde r-r)+(\tilde\lambda-\lambda)(p-r)\|_1\le\epsilon_r+2\epsilon_\lambda$, using $\|p-r\|_1\le 2$. Row-stochastic transports are $L^1$-nonexpansive: $\|\tilde T\tilde a-Ta\|_1\le\|\tilde a-a\|_1+\epsilon_T$. Adding and subtracting $(1-\tilde\rho)a+\tilde\rho Ta$ and applying $\|Ta-a\|_1\le 2$, $\tilde\rho\le\rho_{\max}$ yields the $L^1$ bound. On $\Delta_\gamma^{D-1}$ the map $q\mapsto\mathrm{KL}(u\|q)$ is $1/\gamma$-Lipschitz in $L^1$, giving the KL bound.\hfill$\square$

\subsection{Proposition on retrieval consistency}

\begin{proposition}[Retrieval consistency under Lipschitz successor map]
\label{prop:retrieval-app}
Let $m(h)=\mathbb E[p_{t+1}\mid h_t=h]\in\Delta^{D-1}$ and suppose $\|m(h)-m(h')\|_1\le L\|h-h'\|$. Write $p_{s+1}=m(h_s)+\xi_s$ where $\xi_s$ is the centered residual (so that $\sum_j\xi_s(j)=0$). Let $s^\star(t)=\arg\min_{s<t}\|h_t-h_s\|$, $d_t=\|h_t-h_{s^\star(t)}\|$. If the 1-NN retrieved successor is $r_t=p_{s^\star(t)+1}$, then
\[
\|r_t-m(h_t)\|_1\;\le\;Ld_t+\|\xi_{s^\star(t)}\|_1.
\]
\end{proposition}
The proposition weakens the exact matched-memory assumption in Theorem~\ref{thm:representability}: as the causal memory becomes dense in the relevant representation space so that $d_t\to 0$, the 1-NN retrieval anchor converges to the Bayes successor map $m(h_t)$ up to the residual noise $\xi_{s^\star(t)}$. CAST's multi-head soft-attention retrieval (Section~\ref{sec:method}) is not directly characterized by this statement, but the proposition captures the regime where concentrated retrieval mass localizes around nearby causal contexts.

\subsection{Proposition on simplex closure and rollout drift}

\begin{proposition}[Simplex closure and rollout drift control]
\label{prop:simplex-app}
At every CAST step: (a) if $p_t,r_t\in\Delta^{D-1}$, then $a_t\in\Delta^{D-1}$; (b) if $T_t$ is row-stochastic per Section~\ref{sec:method}, then $T_t a_t\in\Delta^{D-1}$; (c) if $0\le\rho_t\le 1$, then $\hat p_{t+1}=(1-\rho_t)a_t+\rho_t T_t a_t\in\Delta^{D-1}$. For ordered supports with unit ground metric and radius $r$, $W_1(a_t,T_t a_t)\le r$ and $W_1(a_t,\hat p_{t+1})\le\rho_t r$; with the mean-shift budget $|\mu(T_t a_t)-\mu(a_t)|\le B_t$, $|\mu(\hat p_{t+1})-\mu(a_t)|\le\rho_t B_t$.
\end{proposition}
These are the simplex-closure and drift-control properties used in Section~\ref{sec:method}; both follow directly from convexity and row-stochasticity, and the mean bound is the same convex-mixture calculation applied to $\mu(\cdot)$.

\subsection{Queueing corollary and unordered case}

\paragraph{Queue occupancy.} For queue occupancy $p_t(k)=\mathbb P(N_t=k)$, the current histogram can be identical while the load phase differs. In the notation of Section~\ref{sec:method}, set $Z_t\in\{\mathrm{up},\mathrm{down}\}$ and suppose at an aliasing distribution $p^\star$: $u_{\mathrm{up}}=(1-\rho)p^\star+\rho T_{\mathrm{right}}p^\star$, $u_{\mathrm{down}}=(1-\rho)p^\star+\rho T_{\mathrm{left}}p^\star$, with $T_{\mathrm{right}},T_{\mathrm{left}}$ opposite radius-1 local transports. A baseline summary that cannot distinguish phases pays $\mathrm{JS}_{1/2}(u_{\mathrm{up}},u_{\mathrm{down}})$ by Theorem~\ref{thm:aliasing}; if the anchor-only hull excludes these transported successors by $L^1$ margin $\delta$, it additionally pays $\tfrac12\delta^2$ by Theorem~\ref{thm:transport}. CAST avoids both penalties when its causal representation identifies the phase and its transport head realizes $T_{\mathrm{right}},T_{\mathrm{left}}$.

\paragraph{Unordered compositional data.} For unordered supports, set $\rho_z=0$; the transport separation does not apply and the relevant theory is Theorem~\ref{thm:aliasing} plus Theorem~\ref{thm:representability}: if similar visible compositions have different contextual successors, a summary that aliases those contexts must average incompatible successors, while a history-conditioned successor retrieval operator can represent the regime-specific Bayes successor.

\subsection{What the theorems do and do not imply}

The three theorems describe a regime in which CAST's hypothesis class is expressive enough to avoid unavoidable averaging penalties. They are not sample-complexity or optimization guarantees, and they do not imply that every observed CAST advantage in Section~\ref{sec:results} stems from aliasing alone---rollout stability, for example, is also driven by persistence anchoring (Table~\ref{tab:ablation-summary}) and by the realized-mean-shift budget in Proposition~\ref{prop:simplex-app}. The empirical-diagnostic check in Appendix~\ref{app:theory-diagnostics} connects the theoretical regime to benchmark behavior dataset-by-dataset.

\section{Synthetic Latent-Kernel Aliasing Experiment}
\label{app:synthetic-aliasing}

This appendix elaborates on the controlled experiment summarized in Section~\ref{sec:results-synthetic}. The setup instantiates the theorem's hard case directly: the current simplex state $p^\star$ is identical across two equiprobable regimes, while the causal history reveals whether the successor is a left or right radius-1 local transport of the anchor. A fixed-summary forecaster must average incompatible successors, a phase-aware anchor-only comparator cannot leave the no-transport hull, and CAST's local transport class can realize the regime-specific successor.

The numerical results were reported in Table~\ref{tab:synthetic-aliasing-experiment}. Figure~\ref{fig:synthetic-aliasing-profiles} visualizes the construction and the resulting KL gap: the two successors are obvious under causal context, the fixed-summary optimum is exactly their mixture average, and CAST's oracle and trained local-transport head are essentially exact.

\begin{figure}[t]
    \centering
    \includegraphics[width=0.78\linewidth]{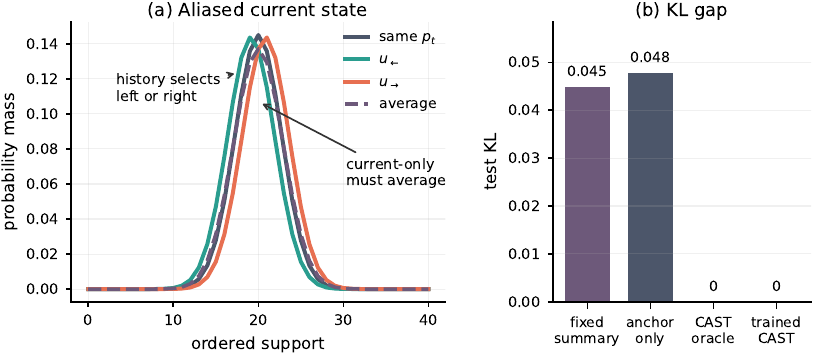}
    \caption{Synthetic latent-kernel aliasing construction. Panel (a): the same current distribution with two history-dependent local-transport successors; a fixed-summary forecaster must collapse them to their average. Panel (b): the resulting KL gap (fixed-summary and anchor-only comparators retain positive error; CAST's oracle and trained local-transport head are essentially exact).}
    \label{fig:synthetic-aliasing-profiles}
\end{figure}

\section{Benchmark Construction and Data Sources}
\label{app:benchmarks}

Each benchmark is converted into regularly spaced probability-vector sequences. The raw observation at time $t$ is first aggregated into nonnegative category mass $x_t\in\mathbb{R}_+^D$ and then normalized to $p_t=x_t/\sum_j x_{t,j}$; rows with no usable mass are excluded. Target positions are adjacent one-step pairs $(p_{1:t},p_{t+1})$ scored through split-specific loss masks (tensor-based public panels) or held-out queue-system manifests (queue benchmark).

\subsection{Source, citation, and license map}

\begin{center}
\scriptsize
\setlength{\tabcolsep}{3.0pt}
\begin{tabular}{@{}p{0.18\linewidth}p{0.40\linewidth}p{0.35\linewidth}@{}}
\toprule
Section & Source & License / reuse note \\
\midrule
BioTIME & BioTIME 2.0 \citep{dornelas2025biotime} & Open access, per-study metadata licenses; cite database and studies. \\
Ember Monthly & Ember monthly electricity data \citep{ember2026monthlyelectricity} & CC-BY-4.0. \\
OWID Dietary & OWID Diet Compositions \citep{owid2023dietcompositions} & CC BY; third-party sources carry own terms. \\
CDC Weekly Deaths & CDC/NCHS weekly deaths \citep{cdc_nchs2023weeklydeaths} & Public Domain U.S. Government; frozen 2023-09-27. \\
BLS QCEW & BLS Quarterly Census of Employment and Wages \citep{bls2026qcew} & Public domain; cite BLS. \\
EPA AirData AQI & EPA AirData daily AQI by county \citep{epa2026airdata} & Public EPA data product. \\
NOAA Storm Events & NOAA/NCEI Storm Events Database \citep{noaa_ncei_stormevents} & U.S. public domain; NCEI toward CC0. \\
NYC TLC Trip & NYC TLC trip records \citep{nyctlc_trip_records} & NYC Open Data; cite source/version. \\
Queue Occupancy & Synthetic $G/G/1$ and $G_t/G/1$ following \citet{di2025transformerqueue,lindley1952theory,harchol2013performance} & Project-generated; cite protocol. \\
\bottomrule
\end{tabular}
\end{center}

\noindent Snapshot URLs and retrieval dates are recorded in the release manifest; they include BioTIME (\texttt{biotime.st-andrews.ac.uk}), Ember (\texttt{files.ember-energy.org/public-downloads/monthly\_full\_release\_long\_format.csv}), OWID (\texttt{ourworldindata.org/grapher/dietary-composition-by-country.csv}), CDC (\texttt{data.cdc.gov/d/muzy-jte6}), BLS QCEW (\texttt{data.bls.gov/cew/data/files/\{year\}/csv/}), EPA AirData (\texttt{aqs.epa.gov/aqsweb/airdata/daily\_aqi\_by\_county\_\{year\}.zip}), NOAA Storm Events (\texttt{ncei.noaa.gov/pub/data/swdi/stormevents/csvfiles/}), and NYC TLC (\texttt{nyc.gov/site/tlc/about/tlc-trip-record-data.page}).

\subsection{Dataset-by-dataset preprocessing}

\paragraph{BioTIME Birds.} BioTIME 2.0 is filtered to \texttt{Birds} species-level rows; sequence unit is a site key (study id, latitude, longitude, depth); time is annual; the longest contiguous annual window is selected, requiring $\ge 15$ observations. The vocabulary is the top-127 species by measure mass plus an \texttt{OTHER} residual ($D=128$, covering $\approx 96.2\%$ of retained mass before the residual). The processed track contains 646 sequences across 37 studies, 15{,}914 site-year observations; per-sequence chronological 70/15/15 gives 10{,}412 train / 2{,}039 val / 2{,}817 test one-step targets.

\paragraph{Ember Monthly.} Monthly full-release long-format CSV normalized across nine detailed-fuel generation categories (\texttt{Coal, Gas, Other Fossil, Nuclear, Hydro, Wind, Solar, Bioenergy, Other Renewables}); $D=9$. 87 country/economy sequences, global raw range 1999-01 to 2026-03, median length 109 monthly observations. Tail-12/12 split when sequence is long enough (84 sequences); 70/15/15 chronological otherwise (3 sequences). Targets: 8{,}272 / 1{,}027 / 1{,}031.

\paragraph{OWID Dietary.} Balanced 1961--2023 annual panel across current ISO3 countries, 25 food-commodity groups; $D=25$, 135 sequences of length 63. Fixed chronological: train 1962--2004, val 2005--2013, test 2014--2023; 5{,}805 / 1{,}215 / 1{,}350 targets.

\paragraph{CDC Weekly Deaths.} Six causes (diseases of heart, malignant neoplasms, cerebrovascular, chronic lower respiratory, COVID-19 underlying, \texttt{OTHER\_CAUSES}); 50 state sequences of length 190. Suppressed 1--9 counts imputed at 5.0; final 4 provisional weeks dropped before splitting; train 2020W02--2021W33, val 2021W34--2022W33, test 2022W34--2023W33; 4{,}250 / 2{,}600 / 2{,}600 targets.

\paragraph{BLS QCEW.} Annual singlefile data 2010--2024 aggregated to state $\times$ supersector (13 categories); $D=13$. 33{,}301 retained rows from $\sim$53.9M raw; 53 state sequences of length 15; per-sequence 70/15/15; 477 / 106 / 159 targets.

\paragraph{EPA AirData AQI.} Daily county AQI 2000--2024 aggregated to county-month distributions over six ordered categories (\texttt{Good, Moderate, USG, Unhealthy, Very Unhealthy, Hazardous}); $D=6$. 7{,}818{,}930 rows retained; 800 sequences, median length 191. Tail-12/12 (796 sequences) or 70/15/15 (4); 136{,}728 / 9{,}584 / 9{,}592 targets. The ordering of categories makes $W_1$ a meaningful diagnostic.

\paragraph{NOAA Storm Events.} Corrected yearly details files 2000--2024 aggregated to state $\times$ month event-type shares over top 31 types plus \texttt{OTHER} ($D=32$); 49 sequences, median length 244. Tail-12/12 split; 9{,}594 / 588 / 588 targets.

\paragraph{NYC TLC Trip.} Yellow taxi trip-record Parquet files (2023) aggregated by pickup hour and calendar week over 265 taxi zones plus \texttt{UNKNOWN} ($D=266$). 24 hour-of-day sequences of length 52--53; final-8-week tail split; 863 / 192 / 192 targets.

\paragraph{Queue Occupancy (Homogeneous, Nonhomogeneous, Combined).} Synthetic single-server queues following \citet{di2025transformerqueue}. Each occupancy distribution is $p_t(k)=\mathbb P(N_t\!=\!k)$ with ordered support, so local transport and $W_1$ are meaningful. Homogeneous ($G/G/1$): inter-arrival/service distributions sampled from seven parametric families (Gamma, Erlang, lognormal, two-normal mixture, hyperexponential, uniform, Weibull) with configuration-specific random parameters, expected utilization restricted to $\approx 0.26$--$0.6$. Nonhomogeneous ($G_t/G/1$): six candidate families (as above except Weibull), arrival-side parameters varying sinusoidally. Combined: pooled. Protocol uses $N=500$ arrivals and $R=10{,}000$ replications per configuration; $\Delta t=1$. The CAST benchmark preserves full trajectories. File-level 70/10/20 split stratified by support-width deciles, seed 42: Homogeneous 7000/1000/2000 systems, Nonhomogeneous 6930/990/1980, Combined 13928/1989/3983. Models see only the \texttt{num\_in\_sys\_*} columns; raw arrival/service moments are held out. For offline one-step queue metrics, a streamed fixed-window payload uses 24 test systems and a 128-step offline window; rollout uses 12 held-out systems, 128-step observed context, and a 64-step autoregressive horizon.

\subsection{Benchmark design rationale}

The suite is designed to support three claims simultaneously. First, distribution-valued forecasting is not a niche formatting issue: the suite spans ecology, energy, diet, mortality, labor, air quality, severe weather, mobility, and queueing---all settings where the forecast object is naturally a composition or probability law. Second, general sequence models can be cross-applied, but the task has structure that generic multivariate TSF models do not directly encode: simplex-valued states, time-respecting successor reuse, persistence as a meaningful anchor, and ordered local mass motion on supports with geometry. Third, the mix of smooth low-frequency panels, high-variance event mixtures, high-dimensional sparse mobility shares, and controlled ordered queue dynamics makes it possible to argue that CAST is not tuned to a single dataset family.

\section{Experimental and Implementation Details}
\label{app:experimental-details}

\subsection{Metrics}

For target $p$ and prediction $q$ we report
\[
\mathrm{KL}(p\|q)=\sum_j p_j\log\tfrac{p_j}{q_j},\qquad
\mathrm{JSD}(p,q)=\tfrac{1}{2}\mathrm{KL}(p\|m)+\tfrac{1}{2}\mathrm{KL}(q\|m),\ m=\tfrac{1}{2}(p+q),
\]
$L^1(p,q)=\sum_j|p_j-q_j|$, Bray--Curtis distance, and (ordered supports) $W_1(p,q)=\sum_j\bigl|\sum_{k\le j}(p_k-q_k)\bigr|$. Prediction and target are renormalized to the simplex before metrics, with $\epsilon=10^{-8}$ for numerical stability. Primary one-step ranking metric is offline KL; primary long-horizon metric is rollout JSD.

\subsection{Splits and rollout horizons}

Non-queue split policies are: BioTIME, BLS QCEW---per-sequence chronological 70/15/15; Ember, EPA AirData AQI, NOAA Storm Events---tail 12/12 if history allows, else 70/15/15; CDC Weekly Deaths---tail 52/52 after dropping the final 4 provisional weeks; OWID Dietary---fixed 1962--2004/2005--2013/2014--2023; NYC TLC Trip---tail 8/8 weeks. Queue split: file-level 70/10/20 by whole queue system, stratified by support-width deciles, split seed 42. Rollout (context / horizon / max examples): BioTIME 16/4/8, Ember 36/12/8, OWID 32/10/8, CDC 52/26/8, BLS 8/2/8, EPA 36/12/8, NOAA 36/12/8, NYC TLC 26/8/8, each queue section 128/64/12.

Target counts after splitting (tensor-based): BioTIME 10{,}412 / 2{,}039 / 2{,}817; Ember 8{,}272 / 1{,}027 / 1{,}031; OWID 5{,}805 / 1{,}215 / 1{,}350; CDC 4{,}250 / 2{,}600 / 2{,}600; BLS 477 / 106 / 159; EPA 136{,}728 / 9{,}584 / 9{,}592; NOAA 9{,}594 / 588 / 588; NYC TLC 863 / 192 / 192. Queue split units are whole queue-system trajectories (see Appendix~\ref{app:benchmarks}).

\subsection{Training protocol}

All trainable models optimize the one-step KL loss on masked adjacent pairs with AdamW, learning rate $3\!\times\!10^{-4}$ (linear warmup to base, constant thereafter), weight decay 0.1, gradient clipping 1.0; 200-step warmup; 10{,}000 iterations for non-queue sections, 12{,}000 for queue sections; evaluation every 250 steps with 20 (non-queue) or 50 (queue) validation batches; batch size 16 (non-queue) or 4 (queue) with gradient accumulation 1 or 2; mixed precision enabled when CUDA bfloat16 is available; no \texttt{torch.compile}. Common architecture defaults: hidden dimension 256, 6 layers, 4 attention heads, dropout 0.1. Model selection uses lowest validation one-step KL. Fit-only baselines (\texttt{persistence}, \texttt{analog\_successor}, \texttt{ilr\_var}, \texttt{compositional\_ets}) are fitted on the training split without SGD.

Dataset-specific block sizes (training block / model max sequence / offline batch / rollout batch): BioTIME 8/64/4/4; Ember 24/384/2/2; OWID 24/64/8/8; CDC 52/192/4/4; BLS 8/16/16/8; EPA 24/192/4/4; NOAA 24/192/4/4; NYC TLC 26/128/4/4; each queue section 128/128/4/4.

\subsection{CAST configuration}

\begin{center}
\small
\setlength{\tabcolsep}{4.2pt}
\begin{tabular}{@{}ll@{}}
\toprule
Quantity & Value \\
\midrule
Retrieval heads $M$ & 2 \\
Retrieval head dimension $d_r$ & 64 \\
Persistence gate initialization ($\lambda_t$) & $\approx 0.55$ \\
Persistence gate bounds $[\lambda_{\min},\lambda_{\max}]$ & $[0.05, 0.95]$ \\
Transport radius & 1 \\
Support basis dimension & 2 \\
Initial transport strength & 0.02 \\
Maximum transport strength $\rho_{\max}$ & 0.20 \\
Base mean-shift budget $\delta_\mu$ & 0.25 \\
Scale-dependent shift budget coefficient $\delta_\sigma$ & 0.10 \\
Structural regularization weight $\lambda_{\mathrm{op}}$ & $5\times 10^{-4}$ \\
Smoothness regularization weight & $10^{-4}$ \\
\bottomrule
\end{tabular}
\end{center}

For unordered supports CAST disables local transport ($\rho_t\equiv 0$) and uses only the persistence--successor anchor; for ordered supports the radius-1 stochastic kernel and bounded transport-strength gate are active.

\subsection{Operator regularizer decomposition}

$\mathcal R_{\mathrm{op}}(T_t,a_t)$ is a weighted sum of four target-free terms: (1) a strength penalty proportional to $\rho_t$; (2) an off-identity penalty $\sum_{j,o\ne 0}K_t(j,o)^2$ weighted by the kernel's deviation from the identity; (3) a neighbor-smoothness penalty $\sum_{j}\|K_t(j,\cdot)-K_t(j+1,\cdot)\|^2$ across adjacent support bins; and (4) a mean-shift penalty $(\Delta\mu_t/B_t)^2$ that discourages the transport head from consuming the entire budget. All terms depend only on the learned transport $T_t$ and the current anchor $a_t$, never on additional future targets.

\subsection{Implementation and compute}

Experiments are implemented in PyTorch. Neural models use distributed data parallel across two GPUs on nodes with A100/H200/L40S GPUs, 16 CPU cores, and $\ge 128$\,GB host memory per GPU; fit-only statistical baselines run on CPU. All methods use identical splits, loss masks, simplex normalization, and metric computation. The resource settings are sufficient for the reported suite; CAST itself has no architectural requirement for this particular compute.

\subsection{Seeds and reporting}

The main tables report seed 42 for model initialization and data shuffling; the queue file-level split also uses seed 42. Under distributed execution, each process uses a deterministic rank-offset seed. Tables in Section~\ref{sec:results} are a broad single-seed benchmark comparison, supplemented by a targeted five-seed CAST robustness study in Appendix~\ref{app:cast-seed-robustness}.

\subsection{Reporting caveats}

(i) Offline one-step evaluation is teacher-forced online evaluation over held-out target positions. (ii) Rollout evaluation is autoregressive and does not feed future ground truth after the context. (iii) Queue offline metrics are computed from deterministic streamed evaluation windows; queue rollout is the primary long-horizon stability evidence. (iv) Main results are single-seed; a targeted CAST seed study is reported separately. (v) Incomplete baselines, if any, would be excluded from that model's average rank on missing sections; in our sweep all 16 methods complete all 11 sections.

\section{Full Quantitative Results}
\label{app:full-results}

This appendix reports the complete per-section tables backing the main paper summaries. Table~\ref{tab:appendix-offline-kl-rank-matrix} and Table~\ref{tab:appendix-rollout-jsd-rank-matrix} give the full rank matrix across the 11 sections and all 16 methods. Table~\ref{tab:appendix-full-benchmark-metrics} gives per-section absolute metric values (offline KL/JSD/$L^1$, rollout KL/JSD/$L^1$). Table~\ref{tab:appendix-ablation-full-metrics} and Table~\ref{tab:appendix-ablation-best-variant-deltas} report the full component-study metrics underlying the ablation summary in Table~\ref{tab:ablation-summary}.


\begin{table*}[t]
\centering
\caption{Appendix rank matrix for offline KL. Lower ranks are better; values are ranks within each benchmark section and Avg. is the 11-section mean rank.}
\label{tab:appendix-offline-kl-rank-matrix}
\scriptsize
\setlength{\tabcolsep}{2.4pt}
\begin{tabular}{@{}lrrrrrrrrrrrr@{}}
\toprule
Method & Bio. & Ember & OWID & CDC & BLS & EPA & NOAA & NYC & Q-Hom. & Q-Non. & Q-All & Avg. \\
\midrule
\textbf{CAST} & \textbf{1} & \textbf{1} & \textbf{1} & \textbf{1} & \textbf{1} & 2 & 2 & \textbf{1} & 2 & \textbf{1} & \textbf{1} & \textbf{1.27} \\
Comp. ETS & 4 & 3 & 2 & 5 & 3 & 13 & 14 & 5 & 3 & 3 & 3 & 5.27 \\
Persistence & 5 & 2 & 3 & 7 & 2 & 12 & 15 & 8 & 4 & 2 & 2 & 5.64 \\
N-HiTS & 2 & 5 & 6 & 2 & 5 & \textbf{1} & \textbf{1} & 2 & 9 & 11 & 14 & 5.27 \\
Informer & 6 & 8 & 8 & 4 & 11 & 6 & 4 & 4 & 6 & 6 & 8 & 6.45 \\
Transformer & 11 & 9 & 7 & 6 & 16 & 3 & 6 & 6 & 11 & 12 & 12 & 9.00 \\
ilr-VAR & 15 & 4 & 4 & 3 & 4 & 14 & 13 & 16 & \textbf{1} & 4 & 4 & 7.45 \\
GRU & 10 & 11 & 12 & 10 & 8 & 4 & 7 & 3 & 12 & 7 & 10 & 8.55 \\
iTransformer & 7 & 10 & 11 & 16 & 13 & 5 & 8 & 11 & 5 & 5 & 5 & 8.73 \\
TCN & 9 & 16 & 10 & 9 & 9 & 9 & 3 & 9 & 7 & 8 & 6 & 8.64 \\
TimeMixer & 3 & 12 & 9 & 13 & 14 & 11 & 5 & 7 & 8 & 10 & 7 & 9.00 \\
TiDE & 8 & 13 & 14 & 14 & 12 & 8 & 10 & 12 & 10 & 9 & 9 & 10.82 \\
LSTM & 13 & 14 & 15 & 8 & 7 & 15 & 11 & 13 & 13 & 15 & 11 & 12.27 \\
Analog succ. & 16 & 7 & 5 & 11 & 6 & 16 & 16 & 15 & 14 & 13 & 13 & 12.00 \\
Autoformer & 12 & 15 & 16 & 12 & 10 & 7 & 12 & 10 & 15 & 14 & 15 & 12.55 \\
DLinear & 14 & 6 & 13 & 15 & 15 & 10 & 9 & 14 & 16 & 16 & 16 & 13.09 \\
\bottomrule
\end{tabular}
\end{table*}

\begin{table*}[t]
\centering
\caption{Appendix rank matrix for rollout JSD. Lower ranks are better; values are ranks within each benchmark section and Avg. is the 11-section mean rank.}
\label{tab:appendix-rollout-jsd-rank-matrix}
\scriptsize
\setlength{\tabcolsep}{2.4pt}
\begin{tabular}{@{}lrrrrrrrrrrrr@{}}
\toprule
Method & Bio. & Ember & OWID & CDC & BLS & EPA & NOAA & NYC & Q-Hom. & Q-Non. & Q-All & Avg. \\
\midrule
\textbf{CAST} & 3 & \textbf{1} & 2 & 8 & \textbf{1} & \textbf{1} & \textbf{1} & \textbf{1} & \textbf{1} & \textbf{1} & \textbf{1} & \textbf{1.91} \\
Comp. ETS & \textbf{1} & 3 & 3 & 2 & 5 & 4 & 12 & 2 & 6 & 3 & 2 & 3.91 \\
Persistence & 2 & 4 & \textbf{1} & \textbf{1} & 2 & 5 & 15 & 4 & 2 & 2 & 3 & 3.73 \\
N-HiTS & 4 & 5 & 5 & 5 & 6 & 6 & 4 & 3 & 5 & 11 & 7 & 5.55 \\
Informer & 9 & 10 & 11 & 4 & 9 & 3 & 8 & 7 & 8 & 5 & 5 & 7.18 \\
Transformer & 10 & 7 & 7 & 7 & 11 & 2 & 5 & 6 & 4 & 10 & 4 & 6.64 \\
ilr-VAR & 15 & 2 & 4 & 3 & 3 & 11 & 6 & 16 & 10 & 9 & 16 & 8.64 \\
GRU & 13 & 8 & 8 & 15 & 8 & 13 & 3 & 8 & 3 & 4 & 8 & 8.27 \\
iTransformer & 5 & 12 & 10 & 14 & 14 & 9 & 9 & 9 & 7 & 8 & 10 & 9.73 \\
TCN & 8 & 13 & 13 & 10 & 12 & 8 & 7 & 10 & 12 & 16 & 15 & 11.27 \\
TimeMixer & 6 & 15 & 9 & 13 & 15 & 14 & 10 & 5 & 14 & 15 & 9 & 11.36 \\
TiDE & 7 & 14 & 12 & 9 & 10 & 10 & 11 & 11 & 9 & 7 & 12 & 10.18 \\
LSTM & 16 & 9 & 14 & 12 & 7 & 7 & 2 & 12 & 13 & 6 & 6 & 9.45 \\
Analog succ. & 12 & 6 & 6 & 11 & 4 & 16 & 16 & 15 & 11 & 12 & 14 & 11.18 \\
Autoformer & 14 & 16 & 15 & 6 & 13 & 15 & 14 & 14 & 16 & 14 & 13 & 13.64 \\
DLinear & 11 & 11 & 16 & 16 & 16 & 12 & 13 & 13 & 15 & 13 & 11 & 13.36 \\
\bottomrule
\end{tabular}
\end{table*}

\begingroup
\scriptsize
\setlength{\tabcolsep}{2.2pt}
\begin{longtable}{@{}llrrrrrrrr@{}}
\caption{Full per-section benchmark results for all completed methods. Lower is better for all metric columns. Rows are ordered by offline KL rank within each benchmark section; RO rank is the rollout JSD rank.}\label{tab:appendix-full-benchmark-metrics}\\
\toprule
Dataset & Method & KL rank & Offline KL & Offline JSD & Offline L1 & RO rank & Rollout KL & Rollout JSD & Rollout L1 \\
\midrule
\endfirsthead
\caption[]{Full per-section benchmark results for all completed methods (continued).}\\
\toprule
Dataset & Method & KL rank & Offline KL & Offline JSD & Offline L1 & RO rank & Rollout KL & Rollout JSD & Rollout L1 \\
\midrule
\endhead
BioTIME & \textbf{CAST} & \textbf{1} & \textbf{0.131296} & 0.0288672 & 0.296733 & 3 & 0.128925 & 0.0323185 & 0.34223 \\
 & N-HiTS & 2 & 0.184371 & 0.0441001 & 0.38581 & 4 & 0.278072 & 0.0652255 & 0.515977 \\
 & TimeMixer & 3 & 0.239845 & 0.0587172 & 0.459311 & 6 & 0.460469 & 0.119345 & 0.710424 \\
 & Comp. ETS & 4 & 0.257355 & 0.0269796 & 0.273036 & \textbf{1} & 0.113159 & \textbf{0.0220623} & 0.264881 \\
 & Persistence & 5 & 0.300852 & 0.0325042 & 0.307927 & 2 & 0.130495 & 0.0258193 & 0.279001 \\
 & Informer & 6 & 0.316669 & 0.073841 & 0.526572 & 9 & 1.03989 & 0.211057 & 0.959278 \\
 & iTransformer & 7 & 0.336148 & 0.0752553 & 0.54116 & 5 & 0.388964 & 0.0958021 & 0.644366 \\
 & TiDE & 8 & 0.391031 & 0.0918333 & 0.587084 & 7 & 0.633637 & 0.155722 & 0.843999 \\
 & TCN & 9 & 0.410141 & 0.0967683 & 0.602273 & 8 & 0.807651 & 0.172641 & 0.882695 \\
 & GRU & 10 & 0.419028 & 0.0935012 & 0.583072 & 13 & 1.4907 & 0.312032 & 1.22414 \\
 & Transformer & 11 & 0.494601 & 0.107203 & 0.651694 & 10 & 1.19113 & 0.253003 & 1.1219 \\
 & Autoformer & 12 & 0.983137 & 0.204876 & 0.952815 & 14 & 4.49785 & 0.470729 & 1.53171 \\
 & LSTM & 13 & 1.07034 & 0.235398 & 1.02791 & 16 & 3.82639 & 0.504036 & 1.63129 \\
 & DLinear & 14 & 1.18268 & 0.267732 & 1.10605 & 11 & 2.09595 & 0.280232 & 1.18052 \\
 & ilr-VAR & 15 & 1.55203 & 0.140215 & 0.699709 & 15 & 9.9853 & 0.502051 & 1.66066 \\
 & Analog succ. & 16 & 3.07778 & 0.271958 & 0.98551 & 12 & 2.96658 & 0.290308 & 1.0624 \\
\addlinespace
Ember Monthly & \textbf{CAST} & \textbf{1} & \textbf{0.043419} & 0.0112432 & 0.163191 & \textbf{1} & 0.02491 & \textbf{0.00712896} & 0.118412 \\
 & Persistence & 2 & 0.0558339 & 0.00708276 & 0.125076 & 4 & 0.0680647 & 0.0190364 & 0.207939 \\
 & Comp. ETS & 3 & 0.0594081 & 0.00811643 & 0.133984 & 3 & 0.0592081 & 0.0165153 & 0.189015 \\
 & ilr-VAR & 4 & 0.0768509 & 0.0112401 & 0.163082 & 2 & 0.0422035 & 0.0114587 & 0.157029 \\
 & N-HiTS & 5 & 0.135026 & 0.0300888 & 0.297387 & 5 & 0.0749478 & 0.0200396 & 0.21404 \\
 & DLinear & 6 & 0.185841 & 0.0501369 & 0.353763 & 11 & 0.42838 & 0.0812699 & 0.510123 \\
 & Analog succ. & 7 & 0.196317 & 0.0171863 & 0.181592 & 6 & 0.124148 & 0.0322352 & 0.272416 \\
 & Informer & 8 & 0.243485 & 0.0562127 & 0.430766 & 10 & 0.20515 & 0.0548298 & 0.42204 \\
 & Transformer & 9 & 0.306071 & 0.0480868 & 0.389961 & 7 & 0.202451 & 0.0379548 & 0.343594 \\
 & iTransformer & 10 & 0.329173 & 0.08132 & 0.558565 & 12 & 0.365359 & 0.0994242 & 0.554701 \\
 & GRU & 11 & 0.331481 & 0.0739683 & 0.510132 & 8 & 0.146972 & 0.0419272 & 0.307166 \\
 & TimeMixer & 12 & 0.356405 & 0.0914943 & 0.621615 & 15 & 0.77084 & 0.198903 & 0.973175 \\
 & TiDE & 13 & 0.482048 & 0.107117 & 0.639066 & 14 & 0.859812 & 0.192836 & 0.86541 \\
 & LSTM & 14 & 0.573672 & 0.123553 & 0.693727 & 9 & 0.16772 & 0.0462562 & 0.325333 \\
 & Autoformer & 15 & 0.849209 & 0.219647 & 1.04387 & 16 & 1.05914 & 0.26427 & 1.18726 \\
 & TCN & 16 & 0.860306 & 0.149058 & 0.811679 & 13 & 0.688121 & 0.181629 & 0.918278 \\
\addlinespace
OWID Dietary & \textbf{CAST} & \textbf{1} & \textbf{0.00909632} & 0.00208483 & 0.0725963 & 2 & 0.0251553 & 0.00547772 & 0.132651 \\
 & Comp. ETS & 2 & 0.00974658 & 0.00197552 & 0.0686436 & 3 & 0.0261259 & 0.00548995 & 0.126743 \\
 & Persistence & 3 & 0.00977941 & 0.00199659 & 0.0687337 & \textbf{1} & 0.0258536 & \textbf{0.00544115} & 0.128837 \\
 & ilr-VAR & 4 & 0.0124073 & 0.00259699 & 0.0862119 & 4 & 0.0415185 & 0.00932013 & 0.194097 \\
 & Analog succ. & 5 & 0.0334084 & 0.00579772 & 0.118409 & 6 & 0.113759 & 0.0186083 & 0.234954 \\
 & N-HiTS & 6 & 0.0572862 & 0.0134366 & 0.224671 & 5 & 0.057456 & 0.0146851 & 0.238726 \\
 & Transformer & 7 & 0.149681 & 0.0354405 & 0.376776 & 7 & 0.111226 & 0.0271152 & 0.330619 \\
 & Informer & 8 & 0.162702 & 0.0363493 & 0.374594 & 11 & 0.217676 & 0.0520433 & 0.460199 \\
 & TimeMixer & 9 & 0.186185 & 0.04415 & 0.442642 & 9 & 0.184034 & 0.0432359 & 0.426927 \\
 & TCN & 10 & 0.221173 & 0.0496117 & 0.446061 & 13 & 0.331781 & 0.081625 & 0.603847 \\
 & iTransformer & 11 & 0.225536 & 0.0511349 & 0.459742 & 10 & 0.208341 & 0.0515749 & 0.482573 \\
 & GRU & 12 & 0.253017 & 0.0538751 & 0.466933 & 8 & 0.145166 & 0.0347956 & 0.382494 \\
 & DLinear & 13 & 0.292235 & 0.0791687 & 0.57811 & 16 & 0.663966 & 0.144923 & 0.822368 \\
 & TiDE & 14 & 0.327727 & 0.0725552 & 0.573166 & 12 & 0.331337 & 0.0767271 & 0.602625 \\
 & LSTM & 15 & 0.363798 & 0.0899221 & 0.648624 & 14 & 0.365886 & 0.0915409 & 0.65642 \\
 & Autoformer & 16 & 0.413582 & 0.0969842 & 0.662566 & 15 & 0.520798 & 0.122952 & 0.762143 \\
\addlinespace
CDC Weekly Deaths & \textbf{CAST} & \textbf{1} & \textbf{0.0107881} & 0.00296787 & 0.079106 & 8 & 0.0582988 & 0.0174321 & 0.185183 \\
 & N-HiTS & 2 & 0.0142264 & 0.00379054 & 0.108123 & 5 & 0.0400412 & 0.0117563 & 0.148261 \\
 & ilr-VAR & 3 & 0.0164484 & 0.00259223 & 0.0781426 & 3 & 0.0187595 & 0.00504297 & 0.116092 \\
 & Informer & 4 & 0.0167048 & 0.00462867 & 0.108243 & 4 & 0.0335239 & 0.00959844 & 0.149116 \\
 & Comp. ETS & 5 & 0.0186036 & 0.00265749 & 0.0807235 & 2 & 0.0140918 & 0.00371193 & 0.099118 \\
 & Transformer & 6 & 0.0195876 & 0.00538896 & 0.105769 & 7 & 0.0458788 & 0.0132165 & 0.167135 \\
 & Persistence & 7 & 0.0202088 & 0.00289444 & 0.0862581 & \textbf{1} & 0.0139517 & \textbf{0.00367584} & 0.0981381 \\
 & LSTM & 8 & 0.0244364 & 0.00692017 & 0.121431 & 12 & 0.0807208 & 0.0240117 & 0.228398 \\
 & TCN & 9 & 0.0287298 & 0.00822043 & 0.13316 & 10 & 0.0759255 & 0.022323 & 0.232646 \\
 & GRU & 10 & 0.0391405 & 0.011486 & 0.148457 & 15 & 0.11474 & 0.0343769 & 0.296633 \\
 & Analog succ. & 11 & 0.0453423 & 0.00556252 & 0.114225 & 11 & 0.0808203 & 0.0236295 & 0.226284 \\
 & Autoformer & 12 & 0.04876 & 0.0145005 & 0.169936 & 6 & 0.0443087 & 0.012723 & 0.174658 \\
 & TimeMixer & 13 & 0.053485 & 0.0158688 & 0.177856 & 13 & 0.103348 & 0.0307619 & 0.274461 \\
 & TiDE & 14 & 0.0539142 & 0.0135576 & 0.230729 & 9 & 0.0661668 & 0.0187953 & 0.22678 \\
 & DLinear & 15 & 0.0600882 & 0.01594 & 0.250935 & 16 & 0.887797 & 0.138453 & 0.850887 \\
 & iTransformer & 16 & 0.0685253 & 0.0202776 & 0.220506 & 14 & 0.110236 & 0.032842 & 0.288916 \\
\addlinespace
BLS QCEW & \textbf{CAST} & \textbf{1} & \textbf{0.000877299} & 0.000210913 & 0.0161524 & \textbf{1} & 0.00101067 & \textbf{0.000245855} & 0.0252409 \\
 & Persistence & 2 & 0.00124199 & 0.000217116 & 0.0148815 & 2 & 0.00101154 & 0.000246087 & 0.0252444 \\
 & Comp. ETS & 3 & 0.00153748 & 0.000287053 & 0.0170236 & 5 & 0.00766992 & 0.00126986 & 0.0322875 \\
 & ilr-VAR & 4 & 0.0017315 & 0.00036926 & 0.02112 & 3 & 0.00131132 & 0.000321105 & 0.0280655 \\
 & N-HiTS & 5 & 0.00223366 & 0.00055218 & 0.0310259 & 6 & 0.00705517 & 0.00146595 & 0.0393494 \\
 & Analog succ. & 6 & 0.00500186 & 0.00122302 & 0.0229677 & 4 & 0.00142995 & 0.000362673 & 0.0260285 \\
 & LSTM & 7 & 0.0184843 & 0.00476039 & 0.111191 & 7 & 0.016034 & 0.00400536 & 0.106318 \\
 & GRU & 8 & 0.0217669 & 0.00540283 & 0.113477 & 8 & 0.0221345 & 0.00559466 & 0.119954 \\
 & TCN & 9 & 0.0246176 & 0.00612776 & 0.130935 & 12 & 0.0485267 & 0.0107716 & 0.180417 \\
 & Autoformer & 10 & 0.0255385 & 0.00648846 & 0.126957 & 13 & 0.0435019 & 0.0110232 & 0.198157 \\
 & Informer & 11 & 0.0261033 & 0.00678146 & 0.121539 & 9 & 0.0279846 & 0.00702364 & 0.141428 \\
 & TiDE & 12 & 0.0279086 & 0.00702397 & 0.148591 & 10 & 0.0328629 & 0.00797709 & 0.158621 \\
 & iTransformer & 13 & 0.0428377 & 0.0110241 & 0.204422 & 14 & 0.0536858 & 0.0135799 & 0.202021 \\
 & TimeMixer & 14 & 0.0675995 & 0.0163482 & 0.212513 & 15 & 0.0652305 & 0.0167058 & 0.220618 \\
 & DLinear & 15 & 0.0889357 & 0.0241089 & 0.293117 & 16 & 0.0965313 & 0.0256036 & 0.315455 \\
 & Transformer & 16 & 0.197784 & 0.0512437 & 0.52214 & 11 & 0.0413023 & 0.0104851 & 0.184597 \\
\addlinespace
EPA AirData AQI & N-HiTS & \textbf{1} & \textbf{0.0951827} & 0.0254094 & 0.261104 & 6 & 0.21827 & 0.063517 & 0.439714 \\
 & \textbf{CAST} & 2 & 0.118364 & 0.0329105 & 0.303775 & \textbf{1} & 0.133078 & \textbf{0.0374855} & 0.333589 \\
 & Transformer & 3 & 0.118412 & 0.0283931 & 0.276533 & 2 & 0.181052 & 0.0446549 & 0.366627 \\
 & GRU & 4 & 0.139433 & 0.0323235 & 0.300142 & 13 & 0.400356 & 0.100397 & 0.646791 \\
 & iTransformer & 5 & 0.165264 & 0.0462373 & 0.384421 & 9 & 0.279152 & 0.0743963 & 0.543373 \\
 & Informer & 6 & 0.167172 & 0.0376505 & 0.328333 & 3 & 0.228875 & 0.0619752 & 0.493508 \\
 & Autoformer & 7 & 0.20381 & 0.0576035 & 0.443592 & 15 & 0.422042 & 0.108073 & 0.711935 \\
 & TiDE & 8 & 0.209028 & 0.0520525 & 0.407176 & 10 & 0.305901 & 0.0777592 & 0.54205 \\
 & TCN & 9 & 0.214709 & 0.0431351 & 0.352139 & 8 & 0.282052 & 0.0734942 & 0.529449 \\
 & DLinear & 10 & 0.254184 & 0.0621291 & 0.371073 & 12 & 0.477053 & 0.0931641 & 0.505303 \\
 & TimeMixer & 11 & 0.272136 & 0.0783938 & 0.537543 & 14 & 0.417715 & 0.10739 & 0.650099 \\
 & Persistence & 12 & 0.331506 & 0.0317229 & 0.290943 & 5 & 1.15134 & 0.0629187 & 0.416975 \\
 & Comp. ETS & 13 & 0.332621 & 0.0323935 & 0.2894 & 4 & 1.0303 & 0.0627821 & 0.409021 \\
 & ilr-VAR & 14 & 0.333877 & 0.0364723 & 0.303371 & 11 & 0.720614 & 0.0899845 & 0.538254 \\
 & LSTM & 15 & 0.337118 & 0.0652318 & 0.440824 & 7 & 0.294233 & 0.0726406 & 0.529403 \\
 & Analog succ. & 16 & 0.499962 & 0.0460878 & 0.354127 & 16 & 2.05277 & 0.143879 & 0.718139 \\
\addlinespace
NOAA Storm Events & N-HiTS & \textbf{1} & \textbf{0.804972} & 0.190844 & 0.884904 & 4 & 1.06746 & 0.250997 & 1.08331 \\
 & \textbf{CAST} & 2 & 1.0042 & 0.240966 & 1.03659 & \textbf{1} & 0.877943 & \textbf{0.204162} & 0.936352 \\
 & TCN & 3 & 1.17051 & 0.28466 & 1.17446 & 7 & 1.36934 & 0.27765 & 1.16312 \\
 & Informer & 4 & 1.18189 & 0.241927 & 1.0171 & 8 & 1.48019 & 0.281181 & 1.14625 \\
 & TimeMixer & 5 & 1.18788 & 0.284153 & 1.16783 & 10 & 1.3886 & 0.304777 & 1.22844 \\
 & Transformer & 6 & 1.25203 & 0.268041 & 1.09856 & 5 & 1.27068 & 0.274235 & 1.13487 \\
 & GRU & 7 & 1.2706 & 0.259499 & 1.07998 & 3 & 1.18076 & 0.239621 & 1.03937 \\
 & iTransformer & 8 & 1.27503 & 0.307094 & 1.23358 & 9 & 1.22876 & 0.296223 & 1.216 \\
 & DLinear & 9 & 1.33342 & 0.282046 & 1.11826 & 13 & 1.71357 & 0.344454 & 1.30853 \\
 & TiDE & 10 & 1.44887 & 0.297881 & 1.17492 & 11 & 1.62198 & 0.314353 & 1.23061 \\
 & LSTM & 11 & 1.4822 & 0.270865 & 1.10645 & 2 & 1.05991 & 0.2161 & 0.970997 \\
 & Autoformer & 12 & 1.92024 & 0.385258 & 1.41546 & 14 & 1.98512 & 0.374555 & 1.39473 \\
 & ilr-VAR & 13 & 2.73715 & 0.246669 & 1.02504 & 6 & 2.83251 & 0.276656 & 1.10898 \\
 & Comp. ETS & 14 & 3.30315 & 0.332319 & 1.23383 & 12 & 3.26815 & 0.341975 & 1.28594 \\
 & Persistence & 15 & 4.3975 & 0.247541 & 0.999917 & 15 & 8.12481 & 0.396987 & 1.39809 \\
 & Analog succ. & 16 & 7.93023 & 0.388703 & 1.35276 & 16 & 8.4979 & 0.413657 & 1.43581 \\
\addlinespace
NYC TLC Trip & \textbf{CAST} & \textbf{1} & \textbf{0.0216781} & 0.00533544 & 0.123575 & \textbf{1} & 0.0392898 & \textbf{0.00962498} & 0.171443 \\
 & N-HiTS & 2 & 0.027418 & 0.00685197 & 0.143953 & 3 & 0.0481726 & 0.0120206 & 0.193679 \\
 & GRU & 3 & 0.0374347 & 0.0093467 & 0.174971 & 8 & 0.0751246 & 0.0178017 & 0.250104 \\
 & Informer & 4 & 0.037457 & 0.00834764 & 0.158408 & 7 & 0.0754962 & 0.0158038 & 0.231958 \\
 & Comp. ETS & 5 & 0.0394307 & 0.00614814 & 0.126547 & 2 & 0.080051 & 0.0116937 & 0.182876 \\
 & Transformer & 6 & 0.0430201 & 0.0110779 & 0.205135 & 6 & 0.0624697 & 0.0151754 & 0.234303 \\
 & TimeMixer & 7 & 0.0453191 & 0.0111672 & 0.203309 & 5 & 0.0615421 & 0.0150163 & 0.222775 \\
 & Persistence & 8 & 0.0461533 & 0.00713079 & 0.138658 & 4 & 0.0968018 & 0.0140436 & 0.210845 \\
 & TCN & 9 & 0.0500653 & 0.012201 & 0.210939 & 10 & 0.171072 & 0.0382772 & 0.380874 \\
 & Autoformer & 10 & 0.132937 & 0.0321771 & 0.351482 & 14 & 5.29464 & 0.313007 & 1.15443 \\
 & iTransformer & 11 & 0.140185 & 0.0325128 & 0.362233 & 9 & 0.0808156 & 0.0208806 & 0.250952 \\
 & TiDE & 12 & 0.141825 & 0.0352235 & 0.424395 & 11 & 0.23163 & 0.0530208 & 0.480762 \\
 & LSTM & 13 & 0.149013 & 0.0376267 & 0.401264 & 12 & 0.263217 & 0.0647972 & 0.558567 \\
 & DLinear & 14 & 1.00539 & 0.247281 & 1.1709 & 13 & 1.4417 & 0.268382 & 1.21295 \\
 & Analog succ. & 15 & 4.02082 & 0.542739 & 1.77242 & 15 & 3.87917 & 0.566316 & 1.81045 \\
 & ilr-VAR & 16 & 13.6058 & 0.691612 & 1.99869 & 16 & 14.6028 & 0.692487 & 1.99953 \\
\addlinespace
Queue Homogeneous & ilr-VAR & \textbf{1} & \textbf{0.0146575} & 0.00257377 & 0.0700608 & 10 & 7.14156 & 0.325373 & 1.04609 \\
 & \textbf{CAST} & 2 & 0.0157613 & 0.00364617 & 0.0543488 & \textbf{1} & 1.32163 & \textbf{0.0732326} & 0.286115 \\
 & Comp. ETS & 3 & 0.0175075 & 0.00326471 & 0.0560934 & 6 & 1.93932 & 0.182566 & 0.673514 \\
 & Persistence & 4 & 0.0202404 & 0.00419852 & 0.0612677 & 2 & 1.5297 & 0.0887025 & 0.336183 \\
 & iTransformer & 5 & 0.0364966 & 0.00796403 & 0.155832 & 7 & 1.78061 & 0.216013 & 0.935441 \\
 & Informer & 6 & 0.0783006 & 0.018033 & 0.220333 & 8 & 1.61026 & 0.24709 & 1.04285 \\
 & TCN & 7 & 0.0815257 & 0.0185813 & 0.221934 & 12 & 3.57281 & 0.336199 & 1.19171 \\
 & TimeMixer & 8 & 0.119072 & 0.0208954 & 0.220727 & 14 & 2.43904 & 0.353977 & 1.30293 \\
 & N-HiTS & 9 & 0.15382 & 0.027131 & 0.195042 & 5 & 1.13843 & 0.173467 & 0.782484 \\
 & TiDE & 10 & 0.190298 & 0.0428669 & 0.312958 & 9 & 2.61056 & 0.317366 & 1.17196 \\
 & Transformer & 11 & 0.191079 & 0.0354214 & 0.272254 & 4 & 1.39686 & 0.16314 & 0.688111 \\
 & GRU & 12 & 0.199407 & 0.0390967 & 0.311613 & 3 & 1.06916 & 0.159669 & 0.74346 \\
 & LSTM & 13 & 0.3267 & 0.0494002 & 0.358521 & 13 & 3.40247 & 0.351077 & 1.19764 \\
 & Analog succ. & 14 & 0.481032 & 0.103656 & 0.400823 & 11 & 4.63087 & 0.330393 & 1.0851 \\
 & Autoformer & 15 & 0.609488 & 0.113984 & 0.664258 & 16 & 11.2235 & 0.483502 & 1.58416 \\
 & DLinear & 16 & 1.69633 & 0.329345 & 1.20749 & 15 & 2.31494 & 0.405561 & 1.39143 \\
\addlinespace
Queue Nonhomogeneous & \textbf{CAST} & \textbf{1} & \textbf{0.0710114} & 0.00666879 & 0.0851155 & \textbf{1} & 0.204651 & \textbf{0.0385748} & 0.318583 \\
 & Persistence & 2 & 0.0795933 & 0.00630115 & 0.0799253 & 2 & 0.349448 & 0.0552407 & 0.374623 \\
 & Comp. ETS & 3 & 0.0823439 & 0.00699556 & 0.0812887 & 3 & 0.405678 & 0.0692842 & 0.425503 \\
 & ilr-VAR & 4 & 0.112021 & 0.00908139 & 0.102818 & 9 & 5.18929 & 0.305106 & 1.02409 \\
 & iTransformer & 5 & 0.141987 & 0.0357759 & 0.309032 & 8 & 1.51932 & 0.298196 & 1.20873 \\
 & Informer & 6 & 0.14358 & 0.0315839 & 0.2856 & 5 & 1.5097 & 0.219169 & 0.972414 \\
 & GRU & 7 & 0.177276 & 0.0374576 & 0.289299 & 4 & 0.630546 & 0.134544 & 0.726834 \\
 & TCN & 8 & 0.211043 & 0.043954 & 0.332068 & 16 & 9.45437 & 0.508572 & 1.62843 \\
 & TiDE & 9 & 0.219794 & 0.0469991 & 0.356809 & 7 & 1.31548 & 0.228863 & 0.988286 \\
 & TimeMixer & 10 & 0.220335 & 0.0407757 & 0.332287 & 15 & 2.91359 & 0.424839 & 1.50168 \\
 & N-HiTS & 11 & 0.264766 & 0.0422817 & 0.302265 & 11 & 1.86893 & 0.354042 & 1.34894 \\
 & Transformer & 12 & 0.328068 & 0.0504815 & 0.370057 & 10 & 2.62586 & 0.322617 & 1.25643 \\
 & Analog succ. & 13 & 0.543859 & 0.114609 & 0.476953 & 12 & 5.31117 & 0.376462 & 1.27208 \\
 & Autoformer & 14 & 1.02563 & 0.177012 & 0.789748 & 14 & 6.00347 & 0.413609 & 1.40729 \\
 & LSTM & 15 & 1.36191 & 0.24386 & 1.00272 & 6 & 1.2865 & 0.221353 & 0.970517 \\
 & DLinear & 16 & 1.45636 & 0.311163 & 1.17229 & 13 & 3.16441 & 0.394726 & 1.39268 \\
\addlinespace
Queue Combined & \textbf{CAST} & \textbf{1} & \textbf{0.00751856} & 0.00183513 & 0.064725 & \textbf{1} & 1.12658 & \textbf{0.0810041} & 0.401138 \\
 & Persistence & 2 & 0.0102985 & 0.00209581 & 0.069828 & 3 & 1.3917 & 0.0952724 & 0.442644 \\
 & Comp. ETS & 3 & 0.0106556 & 0.00217919 & 0.0695206 & 2 & 1.33272 & 0.0882713 & 0.428818 \\
 & ilr-VAR & 4 & 0.0141545 & 0.00297808 & 0.0888672 & 16 & 11.4434 & 0.501139 & 1.50207 \\
 & iTransformer & 5 & 0.10614 & 0.0266371 & 0.285218 & 10 & 2.69081 & 0.347132 & 1.2959 \\
 & TCN & 6 & 0.139265 & 0.0323932 & 0.291559 & 15 & 5.21619 & 0.453163 & 1.53838 \\
 & TimeMixer & 7 & 0.153117 & 0.0307791 & 0.279044 & 9 & 1.9972 & 0.284695 & 1.14236 \\
 & Informer & 8 & 0.163512 & 0.0359912 & 0.324176 & 5 & 1.47397 & 0.179185 & 0.821501 \\
 & TiDE & 9 & 0.267956 & 0.054411 & 0.396357 & 12 & 2.37105 & 0.385379 & 1.35612 \\
 & GRU & 10 & 0.280635 & 0.0447805 & 0.326054 & 8 & 1.80887 & 0.238433 & 0.95875 \\
 & LSTM & 11 & 0.314297 & 0.0595657 & 0.370144 & 6 & 1.6557 & 0.220644 & 0.951859 \\
 & Transformer & 12 & 0.417396 & 0.0625187 & 0.411972 & 4 & 1.72367 & 0.155291 & 0.657375 \\
 & Analog succ. & 13 & 0.508904 & 0.116507 & 0.468033 & 14 & 5.39012 & 0.402995 & 1.31472 \\
 & N-HiTS & 14 & 0.535516 & 0.0386793 & 0.252346 & 7 & 1.53283 & 0.237438 & 1.02648 \\
 & Autoformer & 15 & 1.00371 & 0.149492 & 0.736358 & 13 & 6.58272 & 0.390247 & 1.35194 \\
 & DLinear & 16 & 1.62752 & 0.344405 & 1.26957 & 11 & 2.71501 & 0.364278 & 1.31315 \\
\bottomrule
\end{longtable}
\endgroup

\label{app:ablation-full}
\begingroup
\scriptsize
\setlength{\tabcolsep}{3.0pt}
\begin{longtable}{@{}llrrrr@{}}
\caption{Full CAST component-study metrics by benchmark section. Deltas are variant/CAST$-1$ on the same section; negative values mean the variant has lower error than CAST.}\label{tab:appendix-ablation-full-metrics}\\
\toprule
Dataset & Variant & Offline KL & KL $\Delta$ & Rollout JSD & Rollout $\Delta$ \\
\midrule
\endfirsthead
\caption[]{Full CAST component-study metrics by benchmark section (continued).}\\
\toprule
Dataset & Variant & Offline KL & KL $\Delta$ & Rollout JSD & Rollout $\Delta$ \\
\midrule
\endhead
BioTIME & \textbf{CAST} & \textbf{0.131296} & -- & \textbf{0.0323185} & -- \\
 & w/o structural reg. & 0.131296 & +0.0\% & 0.0323185 & +0.0\% \\
 & anchor only & 0.131296 & +0.0\% & 0.0323185 & +0.0\% \\
 & single-head retrieval & 0.140599 & +7.1\% & 0.0309276 & -4.3\% \\
 & fixed local kernel & 0.131296 & +0.0\% & 0.0323185 & +0.0\% \\
 & w/o persistence mix & 0.146632 & +11.7\% & 0.0280559 & -13.2\% \\
\addlinespace
Ember Monthly & \textbf{CAST} & \textbf{0.043419} & -- & \textbf{0.00712895} & -- \\
 & w/o structural reg. & 0.043419 & +0.0\% & 0.00712895 & +0.0\% \\
 & anchor only & 0.043419 & +0.0\% & 0.00712895 & +0.0\% \\
 & single-head retrieval & 0.0300788 & -30.7\% & 0.00758012 & +6.3\% \\
 & fixed local kernel & 0.043419 & +0.0\% & 0.00712895 & +0.0\% \\
 & w/o persistence mix & 0.133281 & +207.0\% & 0.0138151 & +93.8\% \\
\addlinespace
OWID Dietary & \textbf{CAST} & \textbf{0.00909632} & -- & \textbf{0.00547772} & -- \\
 & w/o structural reg. & 0.00909632 & +0.0\% & 0.00547772 & +0.0\% \\
 & anchor only & 0.00909632 & +0.0\% & 0.00547772 & +0.0\% \\
 & single-head retrieval & 0.0138495 & +52.3\% & 0.00686305 & +25.3\% \\
 & fixed local kernel & 0.00909632 & +0.0\% & 0.00547772 & +0.0\% \\
 & w/o persistence mix & 0.00868301 & -4.5\% & 0.00524196 & -4.3\% \\
\addlinespace
CDC Weekly Deaths & \textbf{CAST} & \textbf{0.0107881} & -- & \textbf{0.0174321} & -- \\
 & w/o structural reg. & 0.0107881 & +0.0\% & 0.0174321 & +0.0\% \\
 & anchor only & 0.0107881 & +0.0\% & 0.0174321 & +0.0\% \\
 & single-head retrieval & 0.00898956 & -16.7\% & 0.0171674 & -1.5\% \\
 & fixed local kernel & 0.0107881 & +0.0\% & 0.0174321 & +0.0\% \\
 & w/o persistence mix & 0.00968232 & -10.2\% & 0.0181817 & +4.3\% \\
\addlinespace
BLS QCEW & \textbf{CAST} & \textbf{0.000877299} & -- & \textbf{0.000245855} & -- \\
 & w/o structural reg. & 0.000877299 & +0.0\% & 0.000245855 & +0.0\% \\
 & anchor only & 0.000877299 & +0.0\% & 0.000245855 & +0.0\% \\
 & single-head retrieval & 0.000742832 & -15.3\% & 0.00024583 & +0.0\% \\
 & fixed local kernel & 0.000877299 & +0.0\% & 0.000245855 & +0.0\% \\
 & w/o persistence mix & 0.00100521 & +14.6\% & 0.000243454 & -1.0\% \\
\addlinespace
EPA AirData AQI & \textbf{CAST} & \textbf{0.118364} & -- & \textbf{0.0374855} & -- \\
 & w/o structural reg. & 0.114322 & -3.4\% & 0.0370734 & -1.1\% \\
 & anchor only & 0.111513 & -5.8\% & 0.0359284 & -4.2\% \\
 & single-head retrieval & 0.11995 & +1.3\% & 0.0398978 & +6.4\% \\
 & fixed local kernel & 0.114509 & -3.3\% & 0.0380775 & +1.6\% \\
 & w/o persistence mix & 0.152539 & +28.9\% & 0.0377938 & +0.8\% \\
\addlinespace
NOAA Storm Events & \textbf{CAST} & \textbf{1.0042} & -- & \textbf{0.204162} & -- \\
 & w/o structural reg. & 1.00412 & +0.0\% & 0.204198 & +0.0\% \\
 & anchor only & 1.00412 & +0.0\% & 0.204198 & +0.0\% \\
 & single-head retrieval & 1.00397 & +0.0\% & 0.208365 & +2.1\% \\
 & fixed local kernel & 1.0042 & +0.0\% & 0.204162 & +0.0\% \\
 & w/o persistence mix & 1.14998 & +14.5\% & 0.209803 & +2.8\% \\
\addlinespace
NYC TLC Trip & \textbf{CAST} & \textbf{0.0216781} & -- & \textbf{0.00962498} & -- \\
 & w/o structural reg. & 0.0216781 & +0.0\% & 0.00962498 & +0.0\% \\
 & anchor only & 0.0216781 & +0.0\% & 0.00962498 & +0.0\% \\
 & single-head retrieval & 0.0217785 & +0.5\% & 0.00985627 & +2.4\% \\
 & fixed local kernel & 0.0216781 & +0.0\% & 0.00962498 & +0.0\% \\
 & w/o persistence mix & 0.0239812 & +10.6\% & 0.0096677 & +0.4\% \\
\addlinespace
Queue Homogeneous & \textbf{CAST} & \textbf{0.0157613} & -- & \textbf{0.0732326} & -- \\
 & w/o structural reg. & 0.0158458 & +0.5\% & 0.0730351 & -0.3\% \\
 & anchor only & 0.0185469 & +17.7\% & 0.0886507 & +21.1\% \\
 & single-head retrieval & 0.01584 & +0.5\% & 0.0746717 & +2.0\% \\
 & fixed local kernel & 0.0160649 & +1.9\% & 0.0712605 & -2.7\% \\
 & w/o persistence mix & 0.0157966 & +0.2\% & 0.0703303 & -4.0\% \\
\addlinespace
Queue Nonhomogeneous & \textbf{CAST} & \textbf{0.0710114} & -- & \textbf{0.0385748} & -- \\
 & w/o structural reg. & 0.0736161 & +3.7\% & 0.0554675 & +43.8\% \\
 & anchor only & 0.0814837 & +14.7\% & 0.0474633 & +23.0\% \\
 & single-head retrieval & 0.0694002 & -2.3\% & 0.0548115 & +42.1\% \\
 & fixed local kernel & 0.0694392 & -2.2\% & 0.124439 & +222.6\% \\
 & w/o persistence mix & 0.0865577 & +21.9\% & 0.0511655 & +32.6\% \\
\addlinespace
Queue Combined & \textbf{CAST} & \textbf{0.00751856} & -- & \textbf{0.0810041} & -- \\
 & w/o structural reg. & 0.00903069 & +20.1\% & 0.109586 & +35.3\% \\
 & anchor only & 0.010207 & +35.8\% & 0.0923865 & +14.1\% \\
 & single-head retrieval & 0.00826227 & +9.9\% & 0.0899711 & +11.1\% \\
 & fixed local kernel & 0.00837703 & +11.4\% & 0.109043 & +34.6\% \\
 & w/o persistence mix & 0.133507 & +1675.7\% & 0.136629 & +68.7\% \\
\bottomrule
\end{longtable}
\endgroup

\begin{table*}[t]
\centering
\caption{Dataset-level best component-control comparison against CAST. The selected component variant is the lowest-error non-CAST CAST control for that metric; negative deltas mean the component control is lower than CAST.}
\label{tab:appendix-ablation-best-variant-deltas}
\scriptsize
\setlength{\tabcolsep}{3.0pt}
\begin{tabular}{@{}lrrrlrrr@{}}
\toprule
Dataset & CAST KL & Best variant KL & KL $\Delta$ & KL variant & CAST RO & Best variant RO & RO $\Delta$ \\
\midrule
BioTIME & 0.131296 & 0.131296 & +0.0\% & w/o structural reg. & 0.0323185 & 0.0280559 & -13.2\% \\
Ember Monthly & 0.043419 & 0.0300788 & -30.7\% & single-head retrieval & 0.00712896 & 0.00712896 & +0.0\% \\
OWID Dietary & 0.00909632 & 0.00868301 & -4.5\% & w/o persistence mix & 0.00547772 & 0.00524196 & -4.3\% \\
CDC Weekly Deaths & 0.0107881 & 0.00898956 & -16.7\% & single-head retrieval & 0.0174321 & 0.0171674 & -1.5\% \\
BLS QCEW & 0.000877299 & 0.000742832 & -15.3\% & single-head retrieval & 0.000245855 & 0.000243454 & -1.0\% \\
EPA AirData AQI & 0.118364 & 0.111513 & -5.8\% & anchor only & 0.0374855 & 0.0359284 & -4.2\% \\
NOAA Storm Events & 1.0042 & 1.00397 & +0.0\% & single-head retrieval & 0.204162 & 0.204162 & +0.0\% \\
NYC TLC Trip & 0.0216781 & 0.0216781 & +0.0\% & w/o structural reg. & 0.00962498 & 0.00962498 & +0.0\% \\
Queue Homogeneous & 0.0157613 & 0.0157966 & +0.2\% & w/o persistence mix & 0.0732326 & 0.0703303 & -4.0\% \\
Queue Nonhomogeneous & 0.0710114 & 0.0694002 & -2.3\% & single-head retrieval & 0.0385748 & 0.0474633 & +23.0\% \\
Queue Combined & 0.00751856 & 0.00826227 & +9.9\% & single-head retrieval & 0.0810041 & 0.0899711 & +11.1\% \\
\bottomrule
\end{tabular}
\end{table*}

\section{CAST Seed Robustness}
\label{app:cast-seed-robustness}

The main benchmark table is a broad single-seed sweep. As a targeted sensitivity check, we rerun CAST over five seeds on the public non-queue sections and compare the CAST seed interval with the immediately adjacent baselines in Table~\ref{tab:appendix-full-benchmark-metrics}.

\begin{table*}[t]
\centering
\caption{Five-seed CAST robustness on the non-queue datasets. Values are mean $\pm$ sample standard deviation over five CAST runs, combining the benchmark run together with robustness seeds 42--45. Lower is better for all metrics. EPA and NOAA offline values use the same rolling-window offline evaluation protocol as the benchmark tables.}
\label{tab:appendix-cast-seed-robustness}
\scriptsize
\setlength{\tabcolsep}{2.2pt}
\resizebox{\textwidth}{!}{%
\begin{tabular}{@{}lrrrr@{}}
\toprule
Dataset & Offline KL & Offline JSD & Rollout KL & Rollout JSD \\
\midrule
BioTIME & $0.129846931 \pm 0.001307300$ & $0.028489747 \pm 0.000339945$ & $0.115820839 \pm 0.013476657$ & $0.029042774 \pm 0.003360500$ \\
Ember Monthly & $0.040094380 \pm 0.006953061$ & $0.010433575 \pm 0.001772240$ & $0.028612465 \pm 0.007330173$ & $0.007923460 \pm 0.001567363$ \\
OWID Dietary & $0.016360002 \pm 0.009978004$ & $0.003982110 \pm 0.002597886$ & $0.036616299 \pm 0.014838105$ & $0.008544794 \pm 0.003986897$ \\
CDC Weekly Deaths & $0.009233608 \pm 0.000985212$ & $0.002518346 \pm 0.000282897$ & $0.062832310 \pm 0.009995790$ & $0.018819080 \pm 0.003049933$ \\
BLS QCEW & $0.000838914 \pm 0.000029311$ & $0.000204769 \pm 0.000003499$ & $0.001010642 \pm 0.000000031$ & $0.000245855 \pm 0.000000003$ \\
EPA AirData AQI & $0.112830553 \pm 0.005006539$ & $0.031074019 \pm 0.001461186$ & $0.134789750 \pm 0.007979795$ & $0.037963682 \pm 0.002572998$ \\
NOAA Storm Events & $0.996603588 \pm 0.031013309$ & $0.237163796 \pm 0.006285711$ & $0.894662467 \pm 0.023715547$ & $0.208050197 \pm 0.004317885$ \\
NYC TLC Trip & $0.021926485 \pm 0.000213208$ & $0.005383274 \pm 0.000042496$ & $0.039566407 \pm 0.000288092$ & $0.009684717 \pm 0.000076063$ \\
\bottomrule
\end{tabular}%
}
\end{table*}

\begin{figure}[t]
    \centering
    \includegraphics[width=\linewidth]{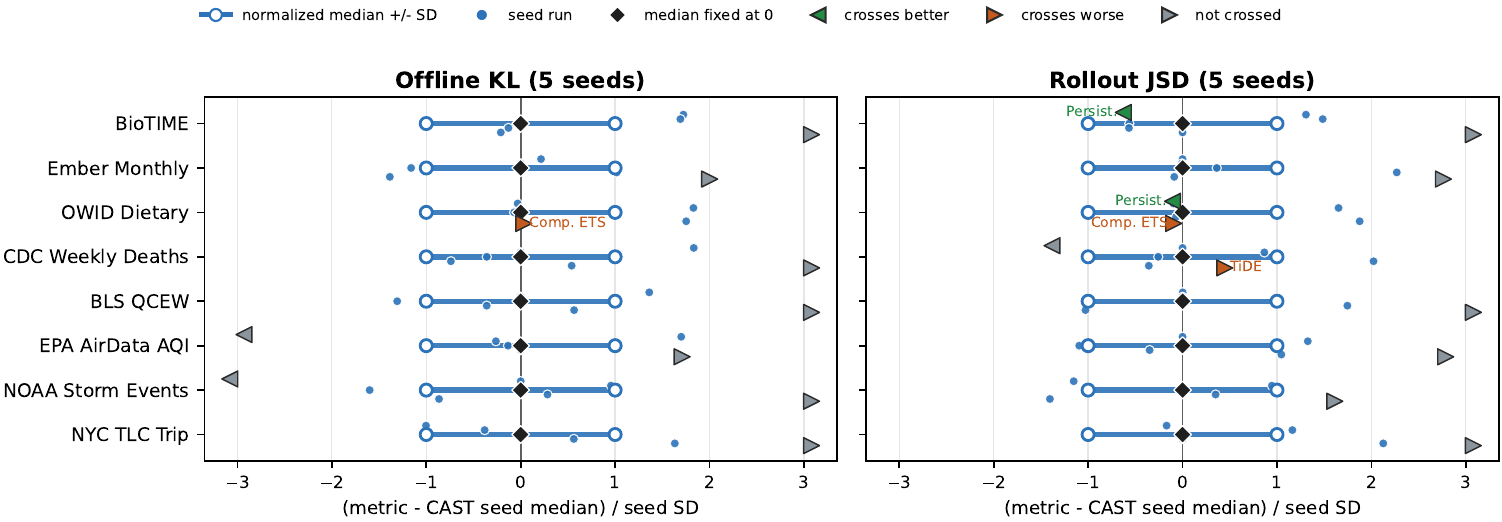}
    \caption{CAST random-seed robustness over five seeds (benchmark seed plus robustness seeds 42--45). Each benchmark row is normalized by CAST seed variability: the horizontal axis is $(\mathrm{metric}-\text{CAST seed median})/\text{CAST seed SD}$, so the CAST interval is fixed at $[-1,1]$ and the seed median is fixed at zero. Triangles mark the relative positions of the immediately better- or worse-ranked appendix-table neighbors (Table~\ref{tab:appendix-full-benchmark-metrics}); green triangles indicate favorable crossings of better-ranked neighbors, orange triangles indicate crossings of worse-ranked neighbors. Baselines beyond the plotted range are clipped at the edge. In the CAST-winning cases the five-seed CAST interval almost always remains separated from the immediately worse baseline; the only unstable exception is OWID Dietary, where the leading methods have very small absolute gaps and the CAST seed spread overlaps its neighbor.}
    \label{fig:cast-seed-robustness-neighbor-overlap}
\end{figure}

\section{Empirical Diagnostics for the Aliasing Theory}
\label{app:theory-diagnostics}

Section~\ref{sec:theory} gives a structural lower bound for fixed-summary forecasters under latent-kernel aliasing. This appendix reports a per-benchmark diagnostic that checks whether the condition behind the theory---similar present distributions but different next-step distributions, under different causal histories---appears in the data.

\subsection{Diagnostic protocol}

We sample states from each benchmark section and, for each sampled state, find cross-sequence neighbors by current distribution; we then measure how close the matched current distributions are (current-neighbor JSD), how different their next-step successors are (successor JSD), and whether a short causal history descriptor gives a closer successor match than the current state alone (history-better rate). For queue variants, the neighbor search uses ordered-support moment and quantile descriptors to propose current-state matches, and we then report actual padded-distribution JSD for the selected pairs. This is a diagnostic proxy rather than proof of the theorem's assumptions: cross-sequence nearest neighbors reveal ambiguity in the map from current distribution to successor, whereas CAST's actual retrieval memory is causal within a sequence; the diagnostic supports the claim that the benchmark contains the same structural failure mode targeted by the theory, not that every ambiguous cross-sequence pair is literally retrieved at test time.

\subsection{Dataset-level summary}

Using a consistent threshold on the neighbor and successor JSD statistics, the benchmark sections partition as follows:

\begin{itemize}
\item \textbf{Strong evidence:} EPA AirData AQI, NOAA Storm Events, Queue Nonhomogeneous, Queue Combined.
\item \textbf{Moderate evidence:} BioTIME, Ember Monthly, CDC Weekly Deaths.
\item \textbf{Weak evidence:} BLS QCEW, NYC TLC Trip, OWID Dietary, Queue Homogeneous.
\end{itemize}

This pattern is consistent with the result pattern of Section~\ref{sec:results}: CAST's rollout-JSD advantage is largest exactly on the strong-evidence sections (top-1 rollout JSD on all four, with $17$--$30\%$ rollout JSD gaps over the closest non-CAST baseline on the two queue sections), moderate on the moderate-evidence sections, and modest or absent on the weak-evidence sections where simple persistence or smooth classical dynamics remain competitive.

\subsection{Queue-specific interpretation}

Queue data make the aliasing picture concrete because the observed state is only the occupancy distribution $\mathbb P(N_t=k)$: arrival intensity, service configuration, utilization, and nonstationary phase are side information or latent context from the model's perspective. The generation protocol of \citet{di2025transformerqueue} uses stationary $G/G/1$ for the homogeneous slice and sinusoidal arrival-side variation for $G_t/G/1$ nonhomogeneous systems; replication averages produce the observed empirical occupancy distributions. The diagnostic supports the following nuanced reading:

\begin{center}
\scriptsize
\setlength{\tabcolsep}{3.0pt}
\resizebox{\linewidth}{!}{%
\begin{tabular}{@{}lrrrrl@{}}
\toprule
Queue slice & Current-nbr JSD med.\ & Successor JSD med.\ & Successor JSD q90 & History-better rate & Interpretation \\
\midrule
Homogeneous & 0.00075 & 0.00085 & 0.00414 & 30.7\% & mostly smooth, weak aliasing \\
Nonhomogeneous & 0.00130 & 0.00264 & 0.01061 & 46.1\% & clearest queue aliasing \\
Combined & 0.00111 & 0.00202 & 0.00851 & 40.6\% & mixed but meaningful ambiguity \\
\bottomrule
\end{tabular}%
}
\end{center}

The nonhomogeneous and combined queue slices are the best empirical matches to the aliasing theory: they often contain very close current occupancy distributions whose successors differ more than the typical one-step drift---the setting where a fixed current-state baseline must average futures and a history-conditioned successor-local operator can do better. The homogeneous slice is different: nearest current matches usually have similar successors, so CAST's queue-homogeneous advantage is better framed as rollout stability and ordered local transport rather than as strong aliasing disambiguation.

\subsection{Summary}

Distribution-valued time series often revisit similar simplex states under different causal contexts. In such cases fixed-summary classical forecasters must collapse distinct successors into a single prediction; CAST avoids this collapse by indexing empirical successors with a causal representation and by applying bounded local transport on ordered supports. The theorem suite identifies a recurring data regime: the same observed simplex state appears under different causal contexts, but the next distribution depends on the context. The diagnostic shows this structure especially clearly in EPA, NOAA, and the nonhomogeneous/combined queue slices, matching the main result pattern.

\section{Additional Qualitative Visualizations}
\label{app:additional-visualizations}

\begin{figure}[t]
\centering
\includegraphics[width=\linewidth]{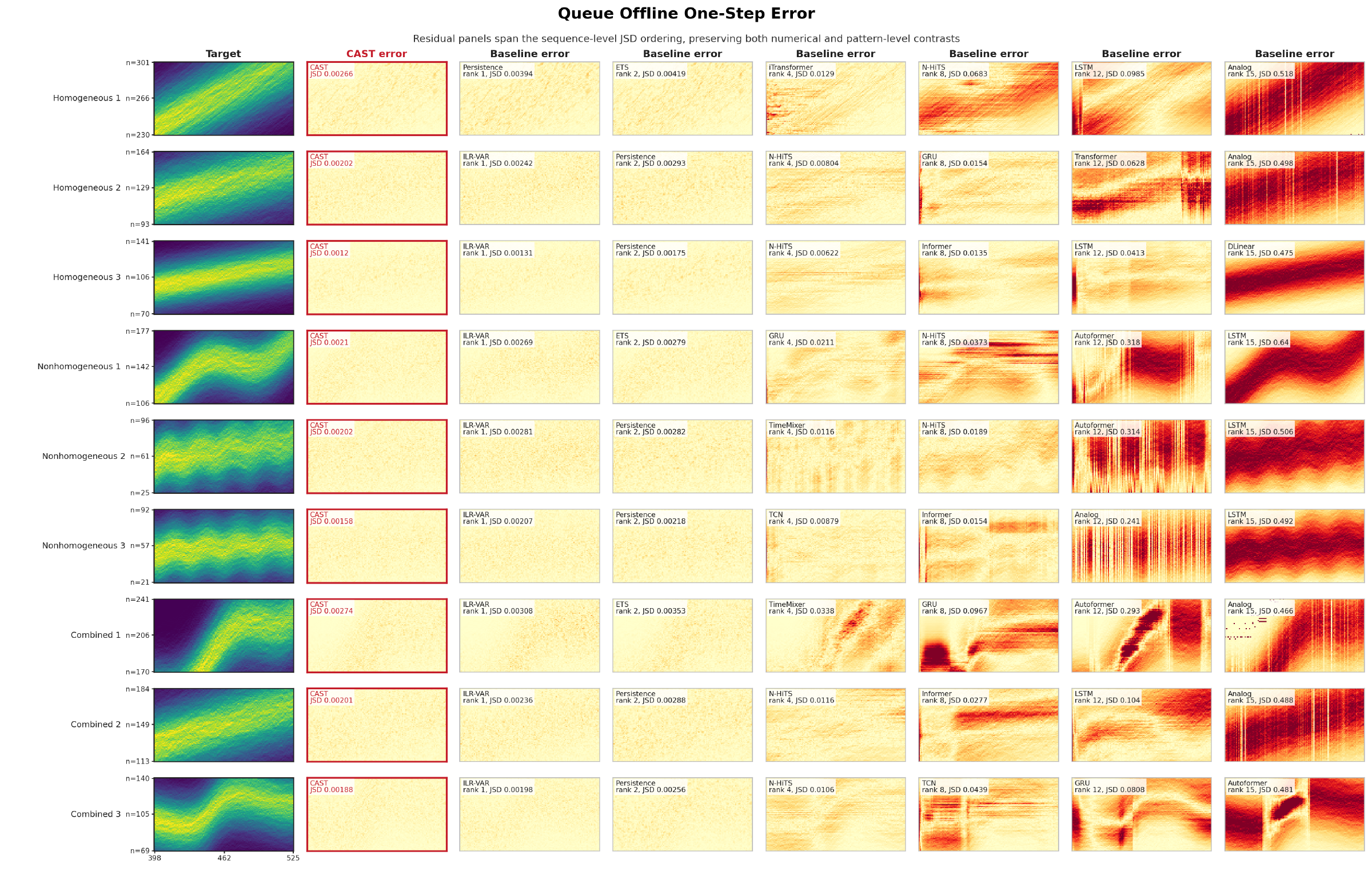}
\caption{Additional offline one-step residual visualizations across homogeneous, nonhomogeneous, and combined queueing regimes. For each target sequence, CAST is compared with baselines spanning a range of sequence-level performance ranks, from numerically competitive methods to substantially weaker predictors. Residual panels show the absolute prediction error and are annotated by the corresponding JSD; lighter and less structured residuals indicate more accurate one-step distribution forecasts. CAST consistently yields diffuse, low-magnitude residuals, whereas many baselines exhibit structured errors aligned with the target ridges, revealing missed temporal shifts even when aggregate scores are close.}
\label{fig:appendix-offline-residual-rank-spread}
\end{figure}

\begin{figure}[t]
\centering
\includegraphics[width=\linewidth]{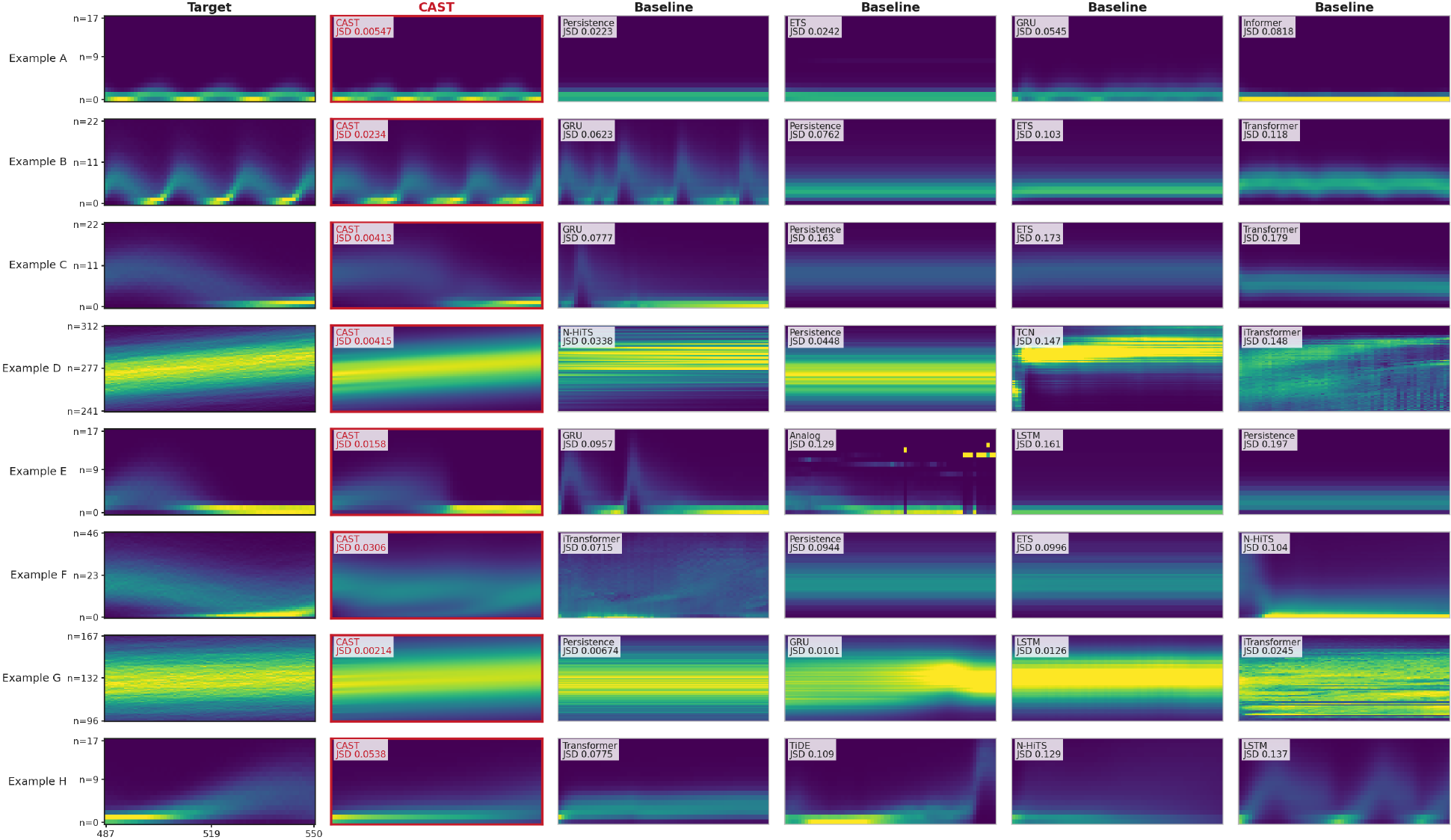}
\caption{Extended rollout visualization for the selected examples. Each row shows the ground-truth queue-state distribution followed by CAST and four representative baselines, annotated by mean rollout JSD. Examples A--D correspond to the main-paper cases, while Examples E--H provide additional held-out systems with different support sizes and temporal dynamics. CAST remains closest to the target evolution across these cases, preserving moving modes and broad distributional shape more reliably than baselines that produce static bands, distorted transitions, or unstable artefacts.}
\label{fig:appendix-rollout-selected-examples}
\end{figure}

\section{Reproducibility and Code Availability}
\label{app:reproducibility}

The code release is available at
\url{https://anonymous.4open.science/r/Causal-Anchored-Simplex-Transport-8775}.
The release contains the CAST implementation, benchmark preprocessing entry points, training and evaluation commands, queue simulation configuration, and instructions for reproducing the reported tables and figures. All experiments use chronological or file-level splits described in Appendix~\ref{app:experimental-details}; public benchmark sources and snapshot notes are listed in Appendix~\ref{app:benchmarks}; queue occupancy data are project-generated from the cited simulation protocol, and the release provides either the simulator and configurations or the processed queue trajectories for reproduction.

\section{Broader Impacts}
\label{app:broader-impacts}

Because CAST targets decision-facing forecasting on aggregate distributions, it is worth discussing positive and negative societal impacts beyond the model-level limitations in Section~\ref{sec:limitations}. Our benchmarks are public aggregate statistical panels and project-generated queueing simulations, so the model itself does not ingest or produce personally identifying data; nevertheless, distribution-valued forecasts can be used as inputs to consequential decisions, which motivates a concrete discussion.

\paragraph{Positive impacts.} CAST is designed for settings where the object of decision is how probability mass over system states evolves. Better distribution-valued forecasts on ordered supports can (i) improve capacity planning for AI and cloud inference infrastructure by anticipating idleness, congestion, and overflow distributions rather than mean load only \citep{mitzenmacher2025queueing}; (ii) support risk-sensitive and distributionally robust decision making by producing probability-law inputs that are directly consumable by CVaR, quantile, and ambiguity-set based planners \citep{li2022quantile,chow2015cvar,wiesemann2014dro,bellemare2017distributional}; (iii) strengthen early-warning pipelines in environmental monitoring (air-quality severity category profiles, storm-event type mixtures) by forecasting the full severity distribution rather than expected values that can hide tail risk; and (iv) aid transportation, energy, and public-health dashboards that already publish share-type indicators (mobility zones, generation-source shares, cause-of-death mixtures), where a distribution-valued forecaster is a more faithful summary than a scalar trend line.

\paragraph{Potential negative impacts and misuse pathways.} CAST is a forecasting operator and does not itself classify individuals, generate content, or take automated actions. The negative-impact pathways are therefore downstream and governance-shaped:
\begin{itemize}
    \item \emph{Distributional forecasts used as automation signals.} If a downstream system automatically allocates resources or triggers interventions from CAST's predicted distributions (e.g., denying service in a region during a predicted congestion spike, triaging care based on predicted mortality-cause mixtures), forecast errors or distributional shift relative to the training snapshot can translate into unfair or unsafe decisions for specific subpopulations. Users should audit forecasts against realized outcomes in deployment and set decision policies that are robust to forecast uncertainty (for example, by coupling CAST with the risk-sensitive planners cited above rather than optimizing expected-value summaries of the forecast).
    \item \emph{Benchmark provenance.} Public panels such as CDC Weekly Deaths, BLS QCEW, EPA AirData, and NOAA Storm Events are publisher-curated aggregates. Their snapshots evolve: update policies, category systems, and revision cycles change over time, so a model tuned on one snapshot may degrade when upstream definitions change. This is not a CAST-specific issue but it matters for any deployment that silently treats new upstream releases as drop-in replacements.
    \item \emph{Latent-group aliasing.} The paper's main theoretical contribution is that fixed-summary forecasters pay an unavoidable $\mathrm{JS}_\pi$ penalty under latent-regime aliasing. The same phenomenon has a fairness angle: if the "latent regime" $Z_t$ correlates with a protected attribute that is not observed at deployment time, even a well-calibrated \emph{average} forecaster may behave unfairly on the regime-conditional subpopulation. CAST's retrieval mechanism helps only to the extent that the causal representation it learns actually separates those regimes; it is not a fairness guarantee in itself.
    \item \emph{Queueing applications.} Queue-distribution forecasts can inform admission control, pricing, or routing policies. Deployments that feed CAST's outputs into such policies should treat the distribution as a planning input rather than a ground-truth, and should preserve human oversight for admission decisions that affect users directly.
\end{itemize}

\paragraph{Mitigations and scope.} We recommend the standard distribution-level forecasting hygiene: (i) monitor realized residuals against forecasts in deployment and retrain or recalibrate when residuals structurally change; (ii) pair CAST with downstream planners that are risk-sensitive or robust to ambiguity, rather than planners that optimize forecast means only; (iii) keep human oversight in the loop whenever the forecast feeds admission, allocation, or triage decisions that affect specific individuals; and (iv) audit how the representation handles latent-group structure by running the diagnostic of Appendix~\ref{app:theory-diagnostics} on deployment data. Because our benchmarks are public aggregate panels and simulated queues, we did not identify a direct high-risk generative or surveillance pathway that would warrant additional model-release safeguards beyond standard scientific documentation and the code release linked in Appendix~\ref{app:reproducibility}.

\end{document}